\documentclass{article}

\usepackage{arxiv}

\usepackage[utf8]{inputenc} 
\usepackage[T1]{fontenc}    
\usepackage{hyperref}       
\usepackage{url}            
\usepackage{booktabs}       
\usepackage{amsfonts}       
\usepackage{nicefrac}       
\usepackage{microtype}      

\usepackage[pdftex]{graphicx}
\usepackage{amsthm}
\usepackage{bm}
\usepackage{mathtools}
\usepackage{subcaption}
\usepackage{multirow}
\usepackage{booktabs}
\usepackage[ruled,vlined]{algorithm2e}
\usepackage{wrapfig}
\newtheorem{theorem}{Theorem}
\newtheorem{corollary}{Corollary}[theorem]

\newtheorem{assumption}{Assumption}
\newtheorem{remark}{Remark}[theorem]
\newtheorem{proposition}{Proposition}

\title{Optimal Gradient Quantization Condition for Communication-Efficient Distributed Training}

\author{
  An Xu, Zhouyuan Huo, Heng Huang \\
  Department of Electrical and Computer Engineering\\
  University of Pittsburgh\\
  \texttt{\{an.xu, zhouyuan.huo, heng.huang\}@pitt.edu} \\
}

\begin{document}
\maketitle

\begin{abstract}
The communication of gradients is costly for training deep neural networks with multiple devices in computer vision applications. In particular, the growing size of deep learning models leads to higher communication overheads that defy the ideal linear training speedup regarding the number of devices. Gradient quantization is one of the common methods to reduce communication costs. However, it can lead to quantization error in the training and result in model performance degradation. In this work, we deduce the optimal condition of both the binary and multi-level gradient quantization for \textbf{ANY} gradient distribution. Based on the optimal condition, we develop two novel quantization schemes: biased BinGrad and unbiased ORQ for binary and multi-level gradient quantization respectively, which dynamically determine the optimal quantization levels. Extensive experimental results on CIFAR and ImageNet datasets with several popular convolutional neural networks show the superiority of our proposed methods.
\end{abstract}


\section{Introduction}

The deep neural networks (DNNs) \cite{krizhevsky2012imagenet,lecun2015deep,szegedy2015going,simonyan2014very} have been successful in resolving many computer vision problems \cite{sermanet2013overfeat,zeiler2014visualizing,he2016deep,huang2017densely,he2016identity}. However, the training process of the large DNN is resource-consuming. Distributed computing is popular for training large DNNs, where the computing nodes need to communicate the gradient of the deep model with each other. The communication overheads are high when the size of parameters becomes large as shown in Table \ref{number of parameters}. In edge computing with federated learning \cite{konevcny2016federated} where edge devices (e.g. mobile phones) send the gradients to the server, the uplink and downlink bandwidth becomes more limited and the reduction of communication is critical. Moreover, the memory consumption can be another concern in edge devices. To address these challenges, gradient quantization directly reduces the size of the gradient compared with 32-bit full precision (FP) floating-point. The most aggressive way is to quantize a gradient value to 1 bit and achieve a compression ratio of 32. TernGrad \cite{wen2017terngrad} and QSGD \cite{alistarh2017qsgd} are two commonly used baseline works to reduce the gradient size. Ternarized gradient uses 3 levels (1.58 bits) as $\{-1,0,+1\}$ to represent the gradient, while QSGD uses 1 bit to represent the sign of the gradient and $\geq 1$ bit to represent the absolute value of the gradient. However, existing works including TernGrad and QSGD either do not take advantage of the statistical structure of the gradient or are usually limited to the assumption of Gaussian or Gaussian-like gradient distribution. Generally speaking, aggressive quantization for a lower communication budget can lead to quantization error in the training and be detrimental to the performance of the model.

In this paper, to improve the performance with quantized gradients, we minimize the variance resulting from quantization via utilizing the gradient distribution information to dynamically select the best quantization levels. Both 2-level (1-bit) and multi-level gradient precision are considered. Our contributions are summarized as follows:

\begin{itemize}
    \item To the best of our knowledge, this is the first work to deduce the optimal condition of 2-level gradient quantization and multi-level gradient random rounding quantization for \textit{any} gradient distribution.
    \item Based on the optimal quantization condition, we propose two novel and practical quantization schemes, binarized gradient (BinGrad-pb/b) and optimized random quantization (ORQ), which dynamically select the optimal quantization levels at \textit{runtime}.
    \item Extensive experimental results on the CIFAR and ImageNet datasets show that our proposed methods achieve better model performance than the counterparts.
\end{itemize}

\begin{table}[t]
\caption{\#Parameter of some commonly used deep models on ImageNet. Comm Time refers to the time to transmit one floating point gradient with 10 Gbps bandwidth.}
\centering
\setlength{\tabcolsep}{4pt}
\begin{tabular}{ccccccc}
\toprule
Model & AlexNet\cite{krizhevsky2012imagenet} & VGG-19 \cite{simonyan2014very} & DenseNet-161 \cite{huang2017densely} & GoogLeNet \cite{szegedy2015going} & ResNet-50 \cite{he2016deep} \\
\midrule
\#Parameter & 61.1 M & 143.7 M & 28.7 M & 13.0 M & 25.6 M\\
Comm Time & 195 ms & 460 ms & 92 ms & 44 ms & 82 ms\\
\bottomrule
\end{tabular}
\label{number of parameters}
\end{table}

\section{Related Works}

Quantization has become a common technique to improve the efficiency of deep learning models. It is most practical for applications with limited hardware resources. We roughly classify it into two categories based on whether the main purpose is to compress the model size for memory efficiency or to reduce the computation and communication complexity for training efficiency.

\paragraph{Weight/activation quantization} The first category is to quantize the model's weights and activations. The activations usually account for a large amount of memory usage on commodity GPUs. Moreover, the arithmetic operations of the floating-point can be replaced with bit-wise operations using quantized weights and activations, which improves the computation efficiency dramatically. Binarized Neural Networks (BNN) \cite{hubara2016binarized} quantizes both the weights and activations to $\{-1, +1\}$ . A straight-through estimator is used to propagate the error gradient through discretization and discard those gradients with large values. But it is only tested on small datasets with small models. \cite{rastegari2016xnor} proposes Binary-Weight-Networks where the weights are quantized to binary values by optimizing the MSE with a scaling factor, and XNOR-Networks by approximating the binary dot product and binary convolution. It scales to ImageNet experiments, although the degradation of top-1 test accuracy up to 12.2\% is witnessed on AlexNet. Some follow-up works include Ternary weight networks \cite{li2016ternary} which constrains the weights to $\{-1, 0, +1\}$, and DoReFa-Net \cite{zhou2016dorefa} which explores the combination of low bitwidth quantization of weights and activations/gradients during the forward/backward pass. To alleviate the performance degradation, \cite{zhang2018lq,jung2019learning} propose to learn the quantizers, \cite{lin2017towards} approximates the full-precision gradient value using a linear combination of binary bases, and \cite{cai2017deep,banner2018post} optimize the quantization levels based on the Gaussian/Laplace distribution assumption.

\paragraph{Gradient quantization} Gradient quantization can be applied during the runtime of the backward pass to avoid expensive arithmetic operations. It also improve memory efficiency and is beneficial in the single machine environment as shown in \cite{rastegari2016xnor, zhou2016dorefa}. Moreover, it is also applied in the distributed environment to reduce the communication cost for distributed training. QSGD splits the whole gradient into buckets of the same length $d$ and independently quantizes each bucket of the gradient. TernGrad compresses the gradient precision with layer-wise ternarizing and employs gradient clipping to reduce the quantization range by removing outlier values. Both of them use random rounding quantization to preserve the unbiased property of the gradient. This is required by optimization theories and critical for distributed training. SignSGD \cite{bernstein2018signsgd}, on the other hand, proposes to use the biased majority vote of the gradient sign in each worker to update the server's parameters under the parameter-server communication architecture. There are some other works such as loss-aware quantization methods \cite{hou2016loss,hou2018loss}, optimizing quantization levels based on Gaussian distribution assumption \cite{he2019optquant}, and theoretical explanation \cite{hou2018analysis} of the effect of gradient clipping as proposed in TernGrad. Besides efforts on more accurate training with low precision gradient, another parallel line of works \cite{lin2017deep,wu2018error,karimireddy2019error} proposes to use the error feedback, where workers locally accumulate the gradient error and then compensate it into the forthcoming batches of training.

We would like to give a remark that some of the works mentioned above can be incorporated along with our methods as reinforcement techniques. For example, we can incorporate the quantized gradient with the gradient sparsification \cite{wangni2018gradient,stich2018sparsified,alistarh2018convergence,basu2019qsparse} technique, where the communication cost is reduced by increasing the sparsity of the gradient to transmit. Consequently in this paper, we are focused on how to optimize the quantization scheme without the interference of other compensational methods.

\section{Optimal Gradient Quantization Condition}

In deep learning, we minimize the non-convex objective function $f(\textbf{x})$ over $N$ data samples, where $\textbf{x}$ is the model parameters to optimize. The objective function can be represented by the sum of the loss function $f_i(\textbf{x})$ regarding the training data $i$:
\begin{equation}
    \min_{\textbf{x}} f(\textbf{x})=\sum^{N}_{i=1} f_i(\textbf{x}).
\end{equation}

Stochastic gradient descent (SGD) randomly samples a data instance $i_t$ at iteration $t$, and does a gradient descent step as in Eq.~(\ref{sgd}) with stochastic gradient $\nabla f_i(\textbf{x})$ instead of full gradient $\nabla f(\textbf{x})$:
\begin{equation} \label{sgd}
\textbf{x}_{t+1}=\textbf{x}_{t}-\gamma_t \mathcal{G}(\textbf{x}_t),
\end{equation}
where $\mathcal{G}(x_t)=\nabla f_{i_t}(\textbf{x}_t)$. 

Let $\mathcal{Q}: \mathbb{R}^d \to \Phi^d$ be the quantization scheme, where $d$ is the number of parameters, $\Phi$ is a subset of $\mathbb{R}$ and also a small finite set. $\mathcal{Q}(\mathcal{G})$ denotes the quantization gradient of the FP gradient $\mathcal{G}$. For simplicity we denote $\mathcal{Q}(v)$ as the quantized value of element $v\in\mathcal{G}$. We start with some basic assumptions in non-convex optimization and reach a relatively direct proposition to guide our efforts by several steps of the proof.

\begin{assumption}
(\textbf{Unbiased gradient}) Both the stochastic gradient $\mathcal{G}(\textbf{x})$ and the quantized stochastic gradient $\mathcal{Q}(\mathcal{G}(\textbf{x}))$ are unbiased estimators of the full gradient:
\begin{equation}
\begin{split}
    \mathbb{E}\mathcal{G}(\textbf{x})=\nabla f(\textbf{x}),\quad\mathbb{E}\mathcal{Q}(\mathcal{G}(\textbf{x}))=\nabla f(\textbf{x}).\\
\end{split}
\end{equation}
\end{assumption}
An ideal design of the quantization scheme $\mathcal{Q}$ should preserve the unbiased property: $\mathbb{E}\mathcal{Q}(\mathcal{G})=\mathcal{G}$.

\begin{assumption}
(\textbf{Bounded variance}) The variance of the FP stochastic gradient is bouned by
\begin{equation}
    \mathbb{E}\left\|\mathcal{G}(\textbf{x})-\nabla f(\textbf{x})\right\|^2\leq \sigma^2.
\end{equation}
\end{assumption}

\begin{assumption}
(\textbf{Lipschitz continuous gradient}) Suppose the objective function $f(\textbf{x})$ is differentialble and have an $L$-Lipschitz continous gradient for some $L>0$:
\begin{equation}\label{lipschitz continous gradient}
    \left\|\nabla f(\textbf{x}_1)-\nabla f(\textbf{x}_2)\right\| \leq L\left\|\textbf{x}_1-\textbf{x}_2\right\|, \forall \textbf{x}_1, \textbf{x}_2 \in \mathbb{R}^d.
\end{equation}
\end{assumption}

\begin{proposition} \label{optimal goal}
The optimal unbiased gradient quantization is to minimize the expected mean square error between quantized gradient and full-precision gradient: $\mathbb{E}\left\|\mathcal{Q}(\mathcal{G}(\textbf{x}))-\mathcal{G}(\textbf{x})\right\|^2$.
\end{proposition}

\begin{proof}
At iteration $t$, $f(\textbf{x}_{t+1})-f(\textbf{x}_t)\leq \left\langle\nabla f(\textbf{x}_t), \textbf{x}_{t+1}-\textbf{x}_t\right\rangle + \frac{L}{2}\left\|\textbf{x}_{t+1}-\textbf{x}_t\right\|^2.$
\begin{equation}
\begin{split}
    &\mathbb{E}f(\textbf{x}_{t+1})-f(\textbf{x}_t)\leq -\gamma_t\left\langle \nabla f(\textbf{x}_t), \mathbb{E}\mathcal{Q}(\mathcal{G}(\textbf{x}_t))\right\rangle + \frac{L\gamma_t^2}{2}\mathbb{E}\left\|\mathcal{Q}(\mathcal{G}(\textbf{x}_t))\right\|^2\\
    & =-\gamma_t\|\nabla f(\textbf{x}_t)\|^2 + \frac{L\gamma_t^2}{2}\mathbb{E}(\|\mathcal{Q}(\mathcal{G}(\textbf{x}_t))-\mathcal{G}(\textbf{x}_t)\|^2+\|\mathcal{G}(\textbf{x}_t)\|^2)\\
    & \leq (\frac{L\gamma_t^2}{2}-\gamma_t)\|\nabla f(\textbf{x}_t)\|^2 + \frac{L\gamma_t^2}{2}(\mathbb{E}\|\mathcal{Q}(\mathcal{G}(\textbf{x}_t))-\mathcal{G}(\textbf{x}_t)\|^2+\sigma^2).\\
\end{split}
\end{equation}
\end{proof}

\subsection{Multi-level Quantization} \label{orq}

For multi-level quantization with $s$ quantization levels $\{b_k\}^{k=\frac{s-1}{2}}_{k=-\frac{s-1}{2}}$, we directly utilize the \textit{random rounding} to keep the unbiased property:
\begin{equation}
\mathcal{Q}(v)=
\begin{cases}
    b_{k-1} \text{ with prob } \frac{b_k-v}{b_k-b_{k-1}},\\
    b_{k} \text{ with prob } \frac{v-b_{k-1}}{b_k-b_{k-1}}.\\
\end{cases}
\end{equation}
where $\mathcal{Q}(v)$ is the quantized value of element $v$ in FP gradient $\mathcal{G}$. In TernGrad the number of quantization levels $s=3$.

However, in both QSGD and TernGrad, $\{b_k\}^{k=\frac{s-1}{2}}_{k=-\frac{s-1}{2}}$ are evenly spaced from $-\|\mathcal{G}\|$ to $\|\mathcal{G}\|$. Constant quantization interval favors gradient with uniform distribution, which is not as observed in experiments \cite{wen2017terngrad}. We may assume the gradient distribution to be Gaussian or Gaussian-like to help the quantization level selection, but still it is not a perfect match. Besides, gradient in different layers usually follows different distribution. Assuming them to be the same is inappropriate. To get the optimal quantization levels for any gradient distribution, we propose the following theorem to guide the optimal levels selection at runtime.

\begin{theorem}
(\textbf{Optimal unbiased quantization}) Suppose $p(v)$ is the distribution of the element $v$ in the FP stochastic gradient $\mathcal{G}$, then the optimal random rounding quantization levels $\{b_k\}$ satisfies,
\begin{equation}
b_{k-1}\int^{b_k}_{b_{k-1}}p(v)dv+b_{k+1}\int^{b_{k+1}}_{b_k}p(v)dv=\int^{b_{k+1}}_{b_{k-1}}vp(v)dv.
\end{equation}
\end{theorem}

\begin{proof}
As suggested in Proposition \ref{optimal goal}, we minimize
\begin{equation}
\begin{split}
\min_{\{b_k\}} D = \mathbb{E}(v-Q(v))^2=\sum_{k}\int^{b_k}_{b_{k-1}}(v-b_{k-1})(b_k-v)p(v)dv.
\end{split}
\end{equation}

Let the gradient regarding $\{b_k\}$ equal to zero,

\begin{equation}
\begin{split}
    \frac{\partial D}{\partial b_k} &= \frac{\partial}{\partial b_k}\int^{b_k}_{b_{k-1}}(v-b_{k-1})(b_k-v)p(v)dv + \frac{\partial}{\partial b_k}\int^{b_{k+1}}_{b_k}(v-b_k)(b_{k+1}-v)p(v)dx\\
    &=\int^{b_{k+1}}_{b_{k-1}}vp(v)dv-b_{k-1}\int^{b_k}_{b_{k-1}}p(v)dv-b_{k+1}\int^{b_{k+1}}_{b_k}p(v)dv=0.\\
\end{split}
\end{equation}
\end{proof}

\begin{corollary} \label{truncated}
For a truncated distribution $p(v)$, we have $l=\min \{v | p(v)=0, \lim_{\epsilon\to 0^+}p(v+\epsilon)>0\}\in\{b_k\}$ and $r=\max\{v | p(v)=0, \lim_{\epsilon\to 0^+}p(v-\epsilon)>0\}\in\{b_k\}$. This means that the range of optimal quantization levels is finite in practice.
\end{corollary}

\begin{proof}
Let $b_{k-1}\leq r\leq b_{k}$. Rearrange the optimal condition and we will have $\int^{r}_{b_{k-1}}(v-b_{k-1})p(v)dv=0$. Thus $b_{k-1}=r$, i.e., $r\in \{b_k\}$. Similarly we also have $l\in\{b_k\}$.
\end{proof}

\begin{remark}
For uniform gradient distribution, the optimal condition becomes $b_k=\frac{1}{2}(b_{k-1}+b_{k+1})$.
\end{remark}

\begin{remark}
We could further simplify the optimal condition as
\begin{equation} \label{simplified optimal condition}
\int^{b_{k+1}}_{b_k}p(v)dv=\frac{\int^{b_{k+1}}_{b_{k-1}}(v-b_{k-1})p(v)dv}{b_{k+1}-b_{k-1}}.
\end{equation}
In particular, for gradient $\mathcal{G}\in \mathbb{R}^D$ with discrete values, Eq.~(\ref{simplified optimal condition}) leads to
\begin{equation}
|\{b_k\leq v\leq b_{k+1}|v\in \mathcal{G}\}|=\frac{\sum_{\{b_{k-1}\leq v \leq b_{k+1}|v\in \mathcal{G}\}}(v-b_{k-1})}{b_{k+1}-b_{k-1}}.
\end{equation}
\end{remark}

\subsection{Binary Quantization} \label{bingrad}

Scaled SignSGD deterministically quantizes the whole gradient as,
\begin{equation}\label{scaled signsgd}
    \mathcal{Q}(\mathcal{G})=\frac{\|\mathcal{G}\|_1}{\text{dim}(\mathcal{G})}\cdot \text{sign}(\mathcal{G}).
\end{equation}

According to Corollary \ref{truncated}, the minimum and maximum quantization levels for random rounding quantization should be the largest and smallest gradient value respectively. It leaves no/little space for the optimization of 2/3 level unbiased quantization. However, using $\{v_{min}, v_{max}\}$ as quantization levels is not resilient to outlier gradient values, leading to large quantization range and non-trivial quantization error. To remove the effect of outliers, we first propose the partially biased binary quantization scheme BinGrad-pb with levels $\{b_{-1},b_1\}$:
\begin{equation}\label{bingrad-pb}
\mathcal{Q}(v)=
\begin{cases}
    b_{-1}, \quad \text{if} \quad v< b_{-1},\\
    b_1, \quad \text{if} \quad v\geq b_1,\\
    b_{-1} \text{ with prob } \frac{b_1-v}{b_1-b_{-1}}, b_{1} \text{ with prob } \frac{v-b_{-1}}{b_1-b_{-1}}, \quad \text{if} \quad b_1<v\leq b_{-1}.\\
\end{cases}
\end{equation}

It is partially biased because for $v\in(b_{-1}, b_1)$ we use the random rounding quantization and the unbiased property is preserved. For BinGrad-pb, we assume that the gradient follows any zero-mean symmetric distribution. We take the gradient of the quantization error $2\int^{b_1}_{0}p(v)(b_1^2-v^2)dv + 2\int^{\infty}_{b_1}p(v)(v-b_1)^2dx$ regarding $b_1$ to zero and reach the following optimal condition:
\begin{equation} \label{bingrad-pb optimal condition}
\begin{split}
    b_1\int^{\infty}_{0}p(v)dv = \int^{\infty}_{b_1}p(v)vdv.
\end{split}
\end{equation}

For gradient $\mathcal{G}\in \mathbb{R}^d$ with discrete values, we minimize the absolute difference of the left-hand side and right-hand side of Eq.~(\ref{bingrad-pb optimal condition}) to get the solution of $b_1$. To further reduce the quantization error, we then propose the fully biased binary quantization scheme BinGrad-b using deterministic quantization:
\begin{equation}\label{bingrad-b}
\mathcal{Q}(v)=
\begin{cases}
b_{-1}, \quad \text{if} \quad  v < b_0,\\
b_{1}, \quad \text{if} \quad v \geq b_0.
\end{cases}
\end{equation}
Then the quantization errors become $\int^{b_0}_{-\infty}p(v)(v-b_{-1})^2dv + \int^{\infty}_{b_0}p(v)(v-b_1)^2dv$. Take its gradient regarding $\{b_{-1},b_0,b_1\}$ to zero and we have the optimal condition for gradient with \textit{any} distribution as
\begin{equation} \label{bingrad-b optimal condition}
    b_0=\frac{b_{-1}+b_1}{2}, \, b_{-1}=\frac{\int^{b_0}_{-\infty}vp(v)dv}{\int^{b_0}_{-\infty}p(v)dv}, \, b_1=\frac{\int^{\infty}_{b_0}vp(v)dv}{\int^{\infty}_{b_0}p(v)dv}.
\end{equation}

From Eq.~(\ref{bingrad-b optimal condition}), $b_{-1}$ and $b_1$ are the mean value of $\{v|v<b_0, v\in\mathcal{G}\}$ and $\{v|v\geq b_0, v\in\mathcal{G}\}$ respectively. BinGrad-b uses deterministic quantization scheme rather than random rounding quantization. We can set $b_0$ in Eq.~(\ref{bingrad-b optimal condition}) to the mean value for ease of implementation. Both the computation complexity of Eqs.~(\ref{bingrad-pb optimal condition},\ref{bingrad-b optimal condition}) for gradient $\mathcal{G}\in \mathbb{R}^D$ is $\mathcal{O}(D)$, which is trivial compared with the training complexity on modern GPUs. In comparison, BinGrad-b achieves optimal quantization error, but introduces some bias; BinGrad-pb reduces the bias but enlarges the quantization error, which leads to a trade-off between bias and variance.

\section{ORQ \& BinGrad Algorithms}

\begin{algorithm}[t]
\caption{ORQ quantization levels via optimal condition Eq.~(\ref{simplified optimal condition}).}
\label{orq quantization levels}
\textbf{Initialize:} \#levels $s=2^K+1 (K=1,2,3,...)$, gradient distribution $p(v)$, $b_{-2^{K-1}}$, $b_{2^{K-1}}$, $l=-2^{K-1}$, $r=2^{K-1}$\;
Denote the solution $b_k$ from Eq.~(\ref{simplified optimal condition}) as $\mathcal{Q}^*(b_{k-1}, b_{k+1}, p)$\;
\textbf{Input:} l, r, p, K\;
\textbf{Output:} $\{b_k\}_{k\in\{l,l+1,...,r\}}=\mathcal{Q}_g(l,r,p,K)$\;
$b_{\frac{l+r}{2}}=\mathcal{Q}^*(b_l, b_r, p)$\;
\If{$K>1$}{
    $\{b_k\}_{k\in\{l,l+1,...,\frac{l+r}{2}\}}=\mathcal{Q}_g(l,\frac{l+r}{2},p,K)$\;
    $\{b_k\}_{k\in\{\frac{l+r}{2},\frac{l+r}{2}+1,...,r\}}=\mathcal{Q}_g(\frac{l+r}{2},r,p,K)$\;
}
\end{algorithm}

\begin{algorithm}[t]
\caption{Distributed SGD training with ORQ/BinGrad-b/BinGrad-pb.}
\label{distributed training}
\textbf{Initialize:} Model parameters $\textbf{x}_0$, $L$ workers (server included), learning rate $\{\gamma_t\}_{t=0}^{T-1}$\;
\For{$t=0,1,\cdots,T-1$}{
    Compute stochastic gradient $\mathcal{G}_t^l$\;
    Compute $\{b_k\}$ from Alg.~\ref{orq quantization levels}/Eq.~(\ref{bingrad-b optimal condition})/Eq.~(\ref{bingrad-pb optimal condition})\;
    Quantize $\mathcal{G}_t^l$ with random rounding/Eq.~(\ref{bingrad-b})/Eq.~(\ref{bingrad-pb})$\to \hat{\mathcal{G}}_t^l$\;
    Encode $\hat{\mathcal{G}}_t^l$ and send to server\;
    \If{is server}{
        Receive and decode $\{\hat{\mathcal{G}}_t^l\}_{l=1}^{L}$ from workers\;
        Broadcast $\bar{\mathcal{G}}_t=\sum^{L}_{l=1}\hat{\mathcal{G}}_t^l/L$ to workers\;
    }
    Update parameters $\textbf{x}_{t+1}\leftarrow \textbf{x}_{t}-\gamma_t \bar{\mathcal{G}}_t$\;
}
\end{algorithm}

Here we introduce the distributed training algorithm with Optimized Random Quantized Gradient (ORQ) and BinGrad. Although on commercial clusters it can be conducted in a decentralized ring-based all\_reduce manner without the server, parameter-server architecture can be applied to more general and practical scenarios, e.g. mobile devices which have strict requirements on communication. For ORQ we propose to use the greedy recursive Algorithm \ref{orq quantization levels} to determine the quantization levels. We also summarize the distributed training method with ORQ/BinGrad-b/BinGrad-pb in Algorithm \ref{distributed training}. If the averaged gradient at the server node needs to be low-level to accelerate the broadcast operation, we can add a step to a) keep $\{b_k\}$ coherent among all the worker nodes; or b) quantize the averaged gradient that the server sends back to the workers. In the single machine environment, we directly update the parameters with the quantized gradient from different quantization schemes.

\section{Experiments}

\begin{figure}[t]
    \centering
    \includegraphics[width=.24\linewidth]{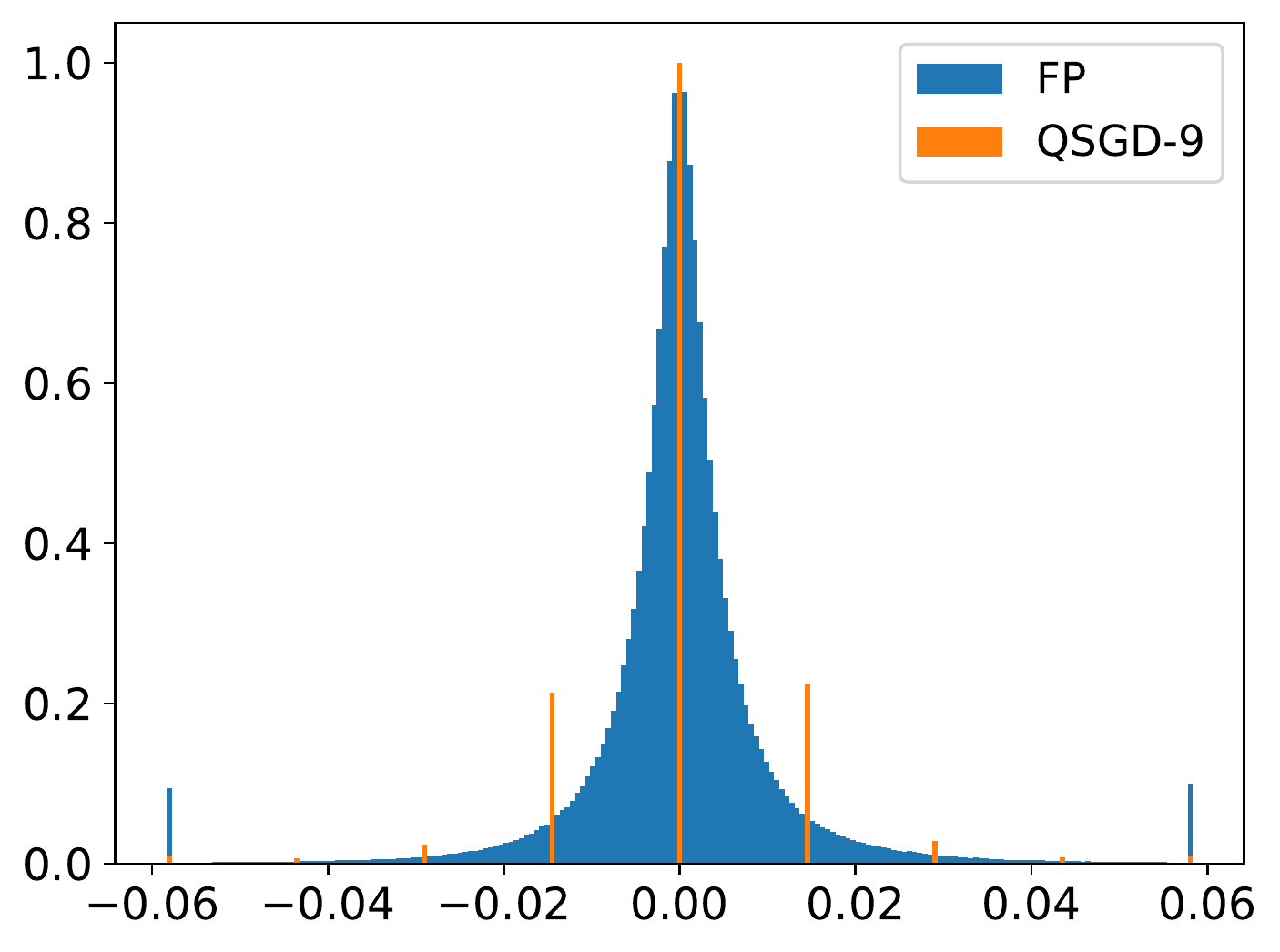}
    \includegraphics[width=.24\linewidth]{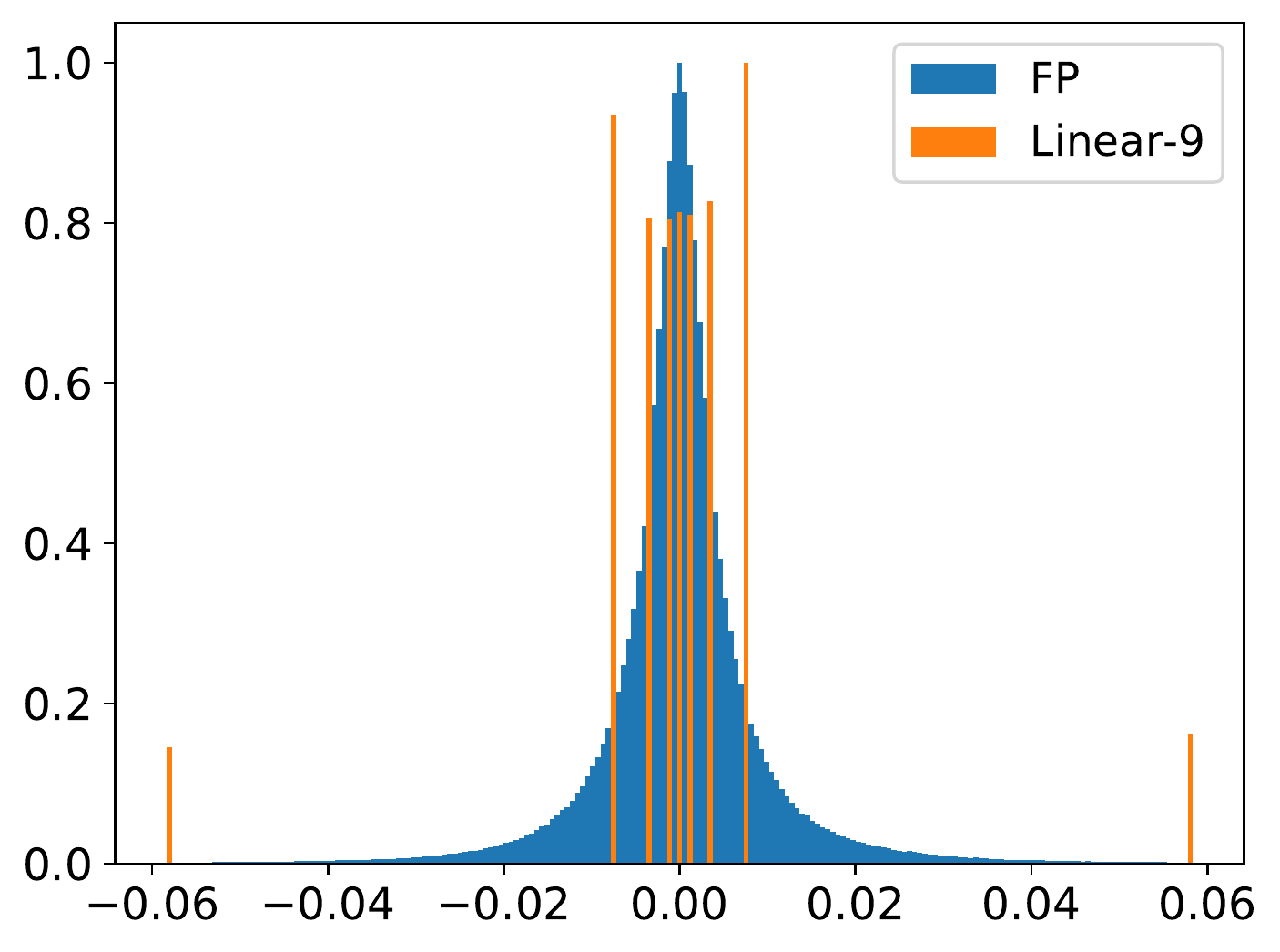}
    \includegraphics[width=.24\linewidth]{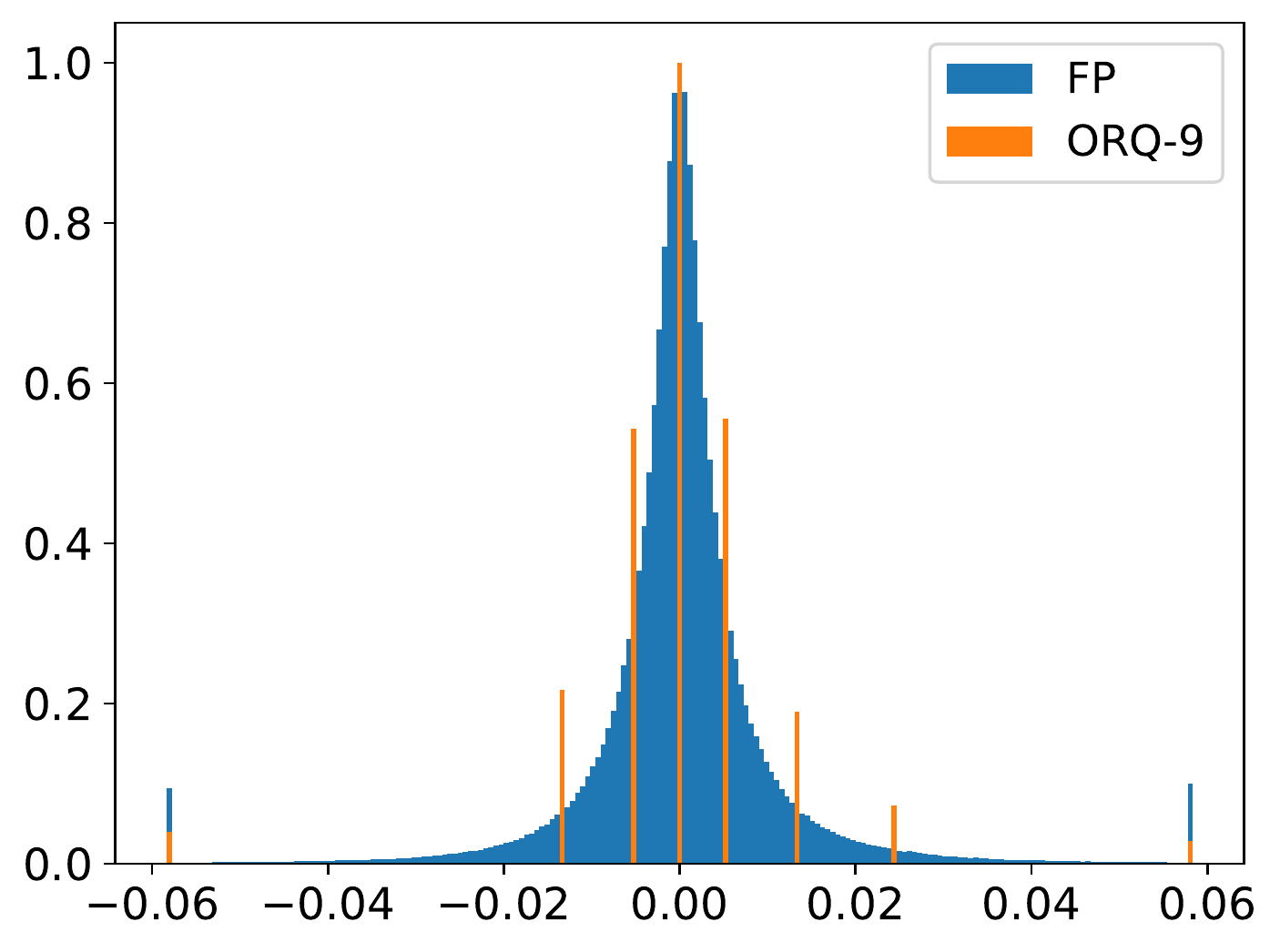}
    \includegraphics[width=.24\linewidth]{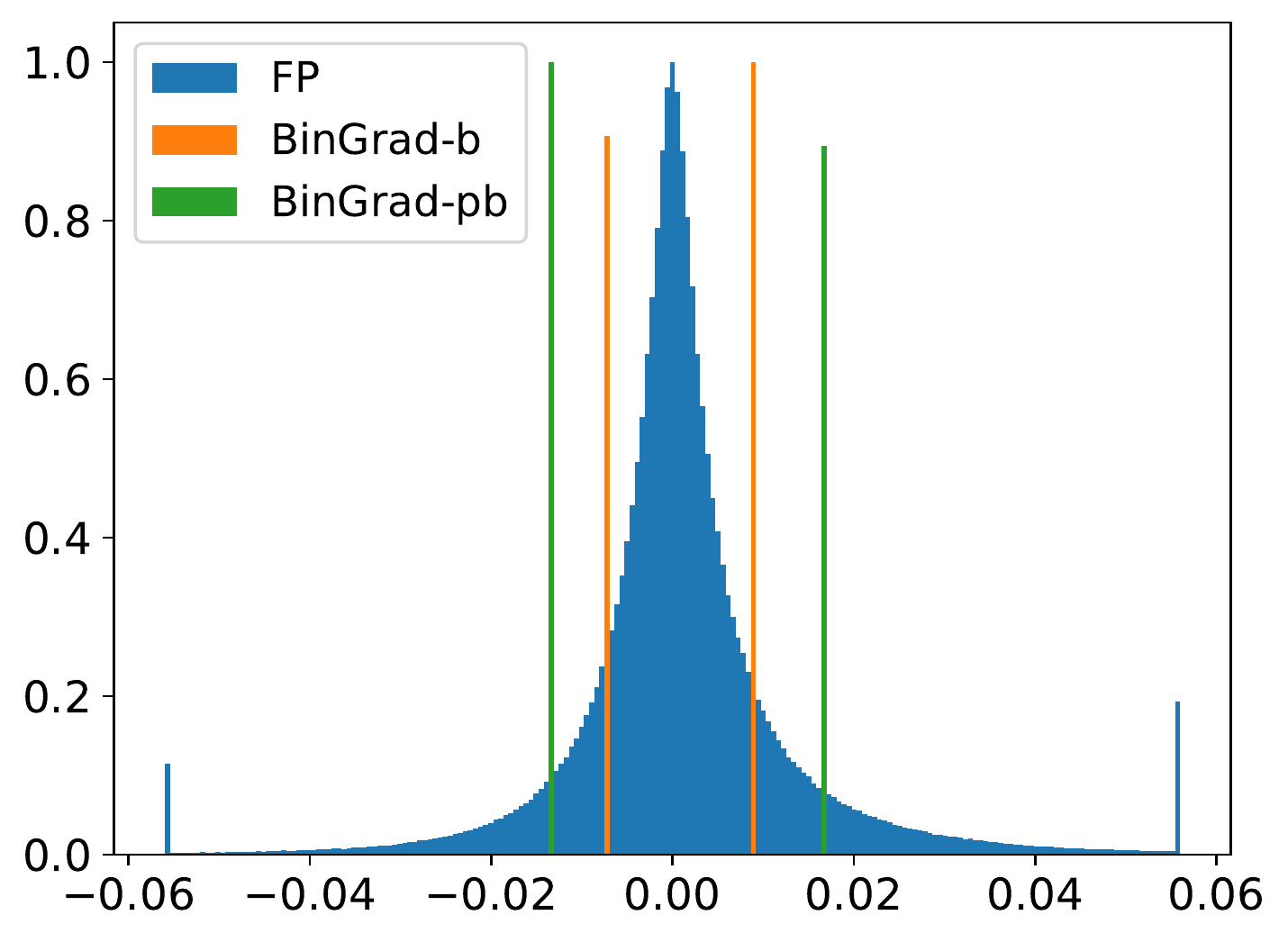}
    \caption{Gradient distribution (CIFAR-10, ResNet-110) of different methods. X axis represents the gradient value and Y axis represents the frequency normalized by the maximum value of the bins in the histogram. Full precision (FP) gradient is clipped into range (-2.5$\sigma$, 2.5$\sigma$).}
    \label{levels distribution}
\end{figure}

\begin{figure}[t]
    \centering
    \includegraphics[width=.32\linewidth]{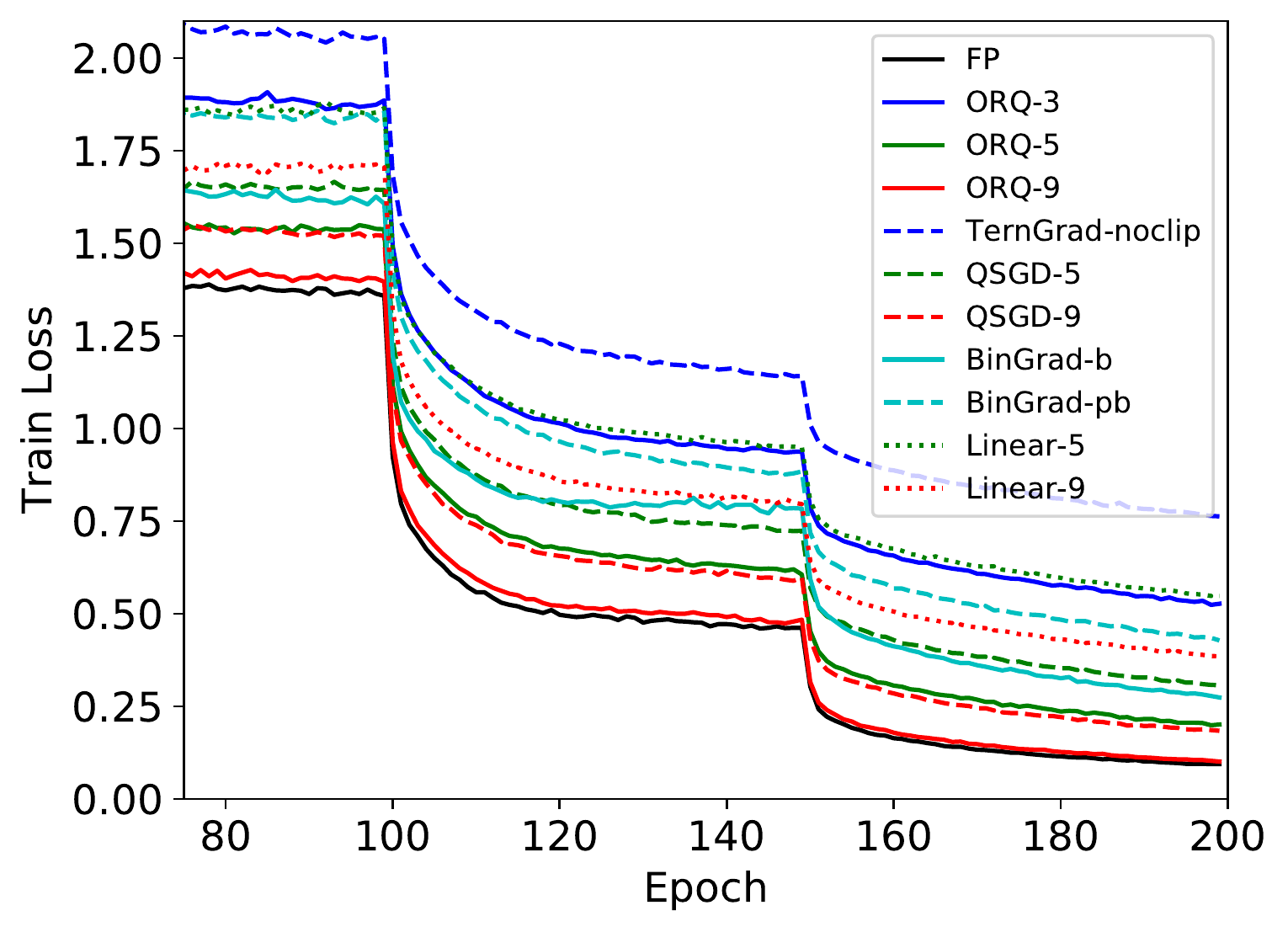}
    \includegraphics[width=.32\linewidth]{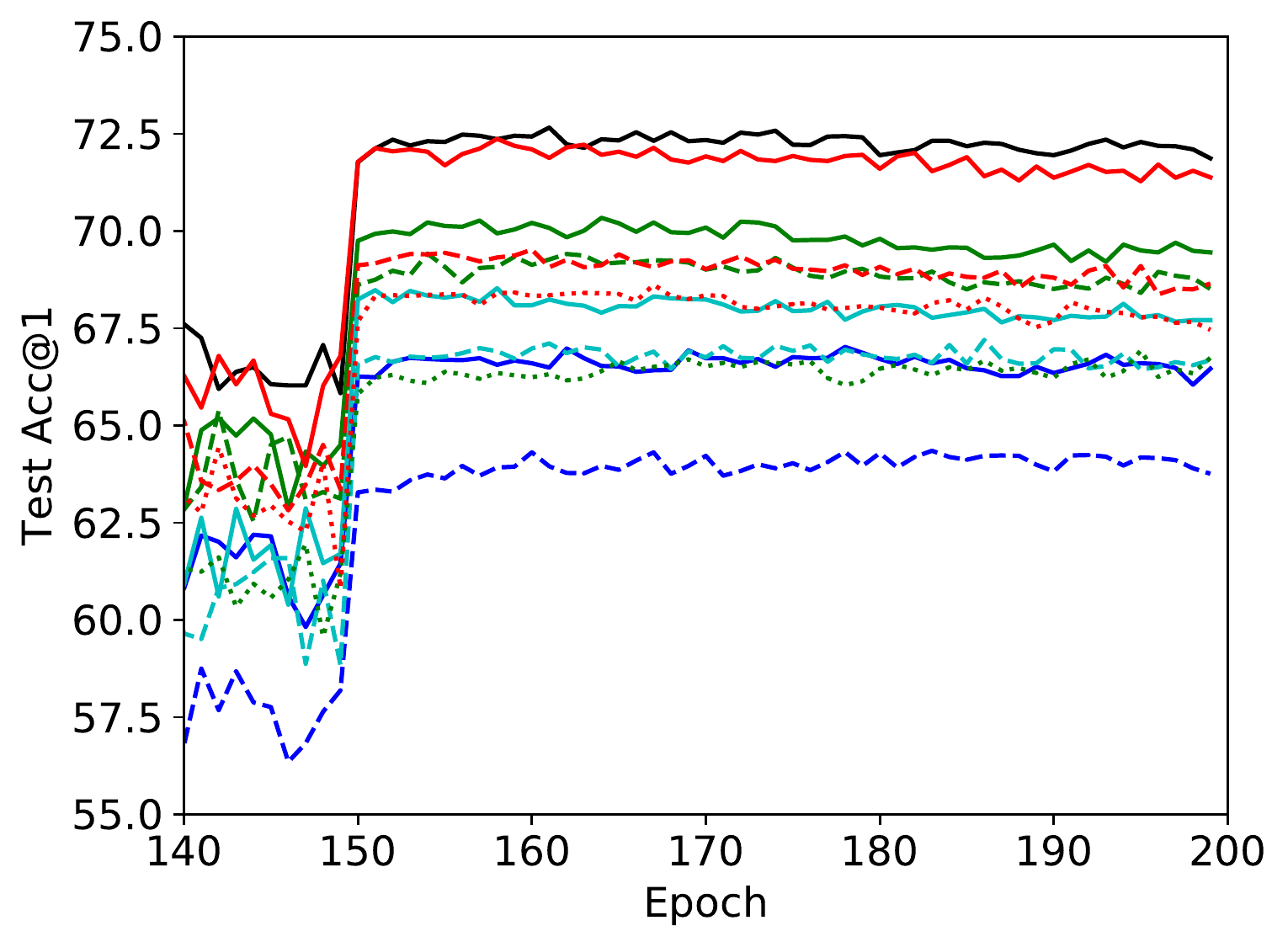}
    \includegraphics[width=.32\linewidth]{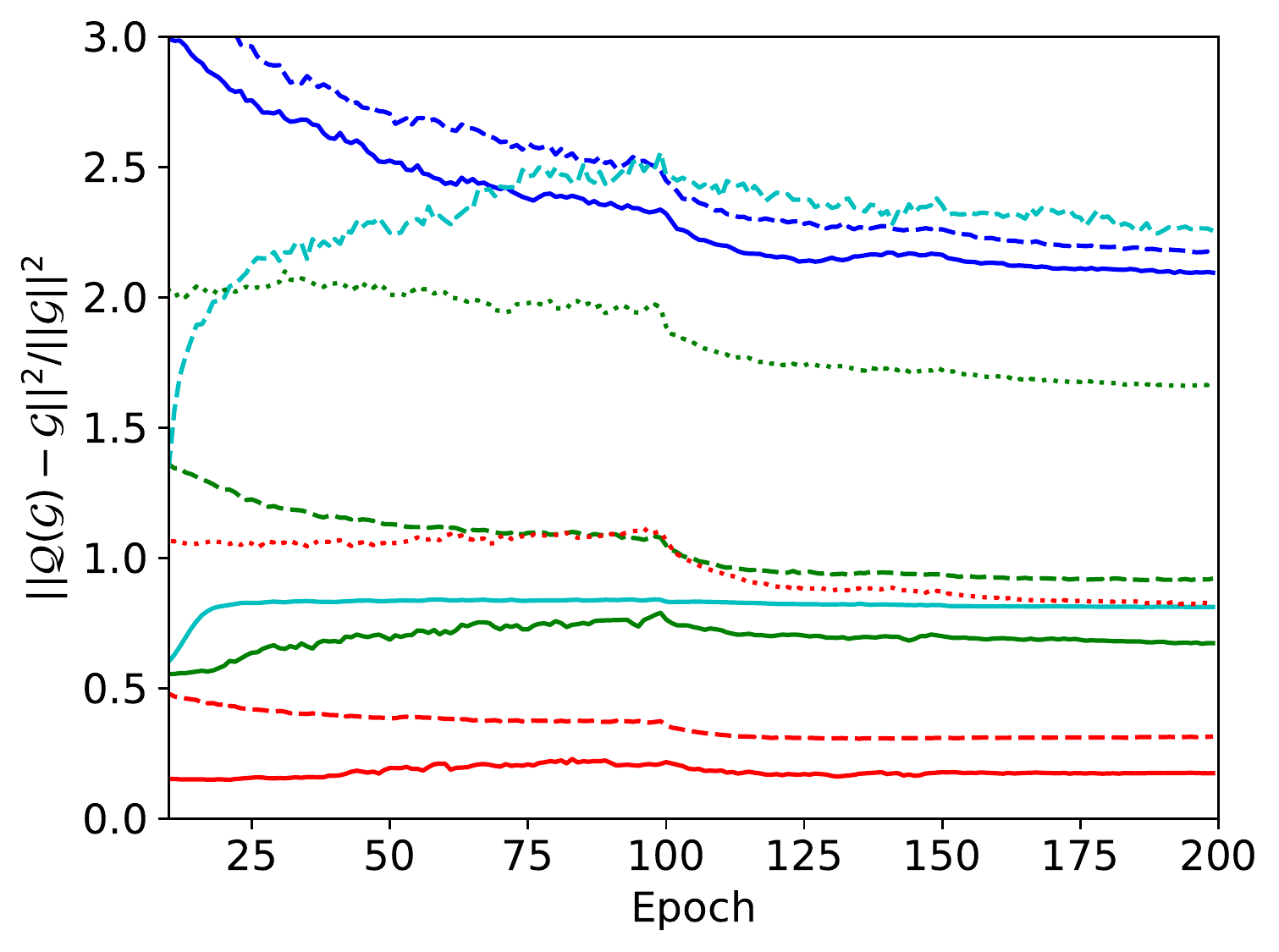}
    \includegraphics[width=.32\linewidth]{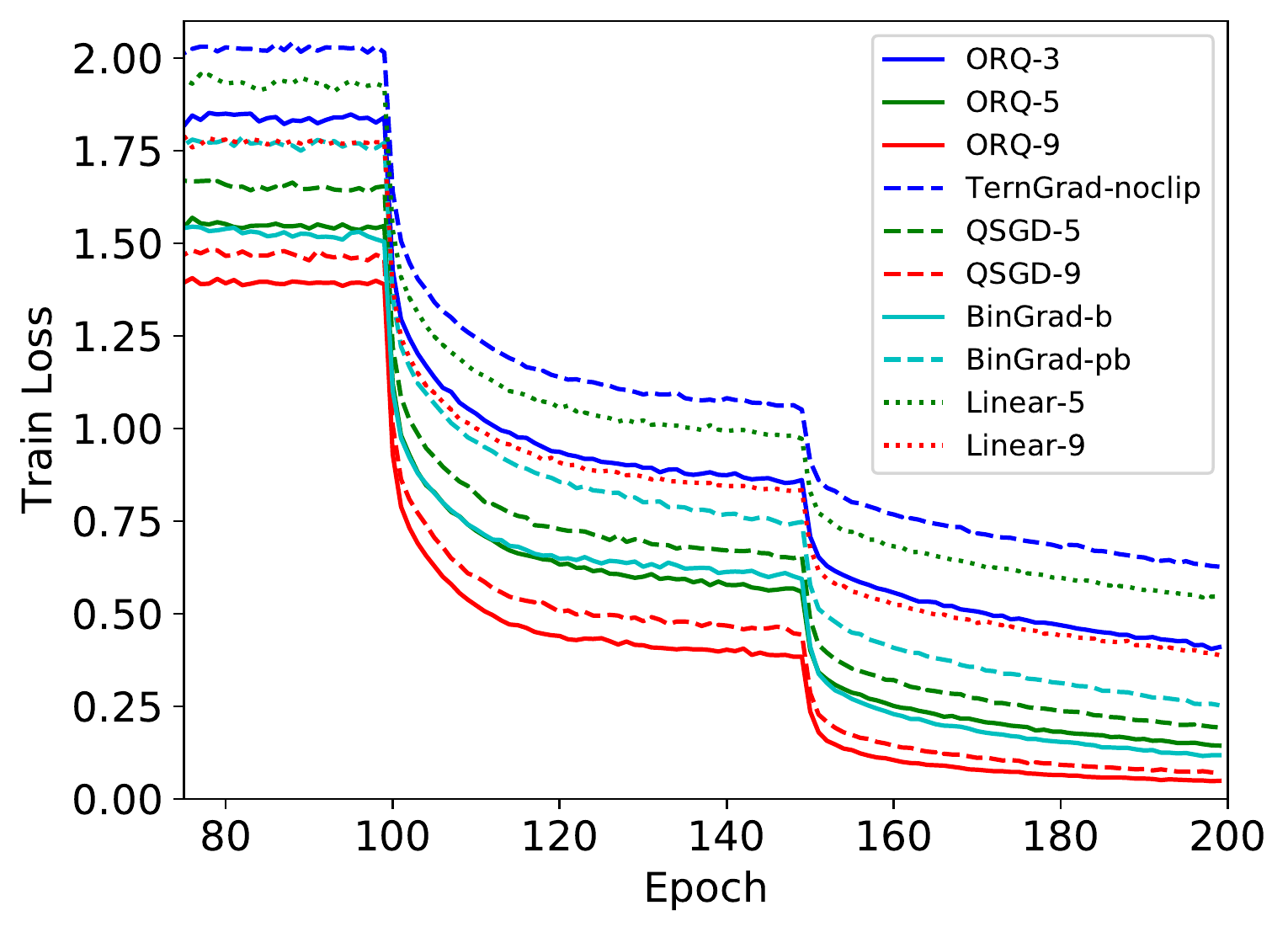}
    \includegraphics[width=.32\linewidth]{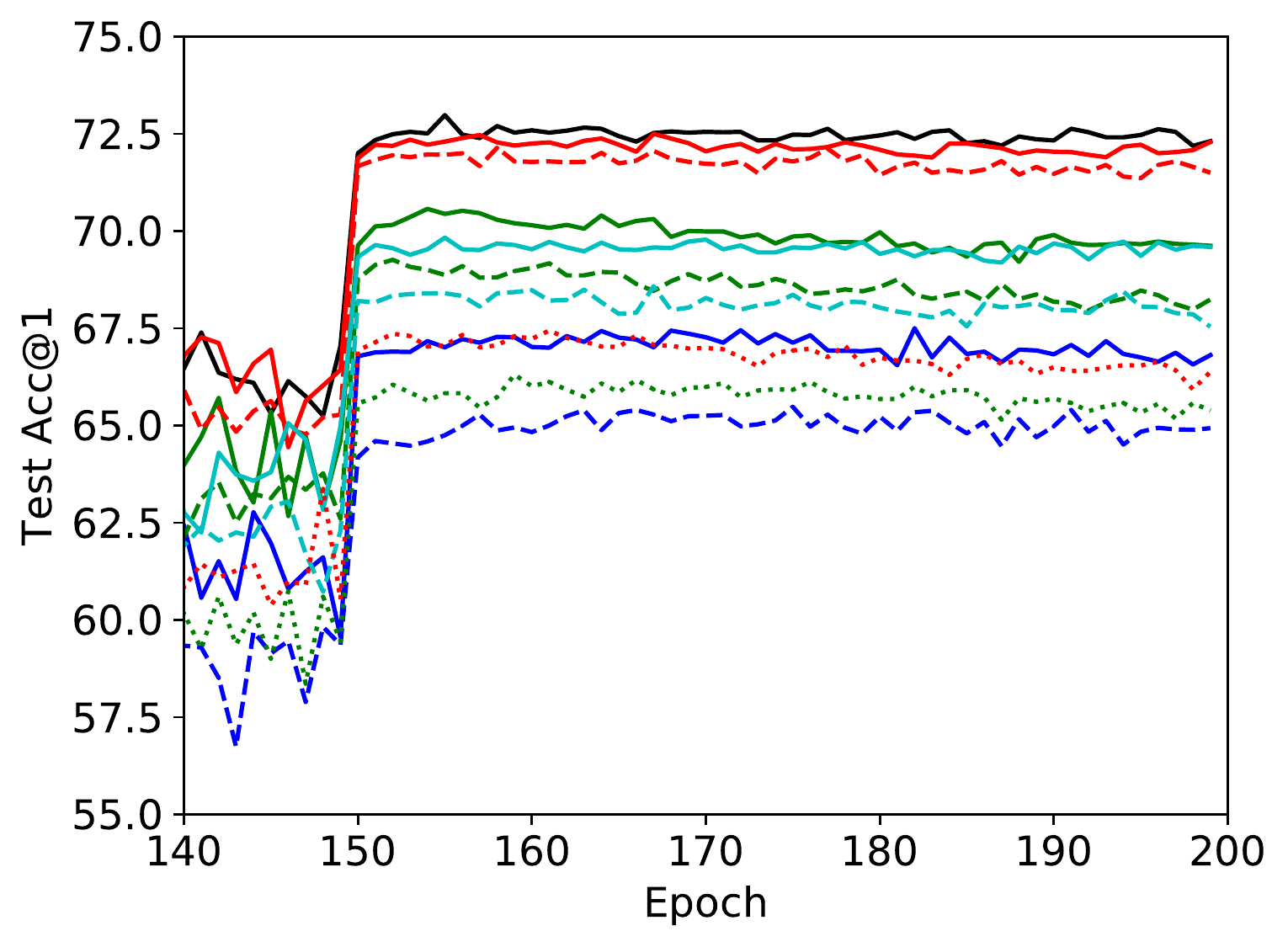}
    \includegraphics[width=.32\linewidth]{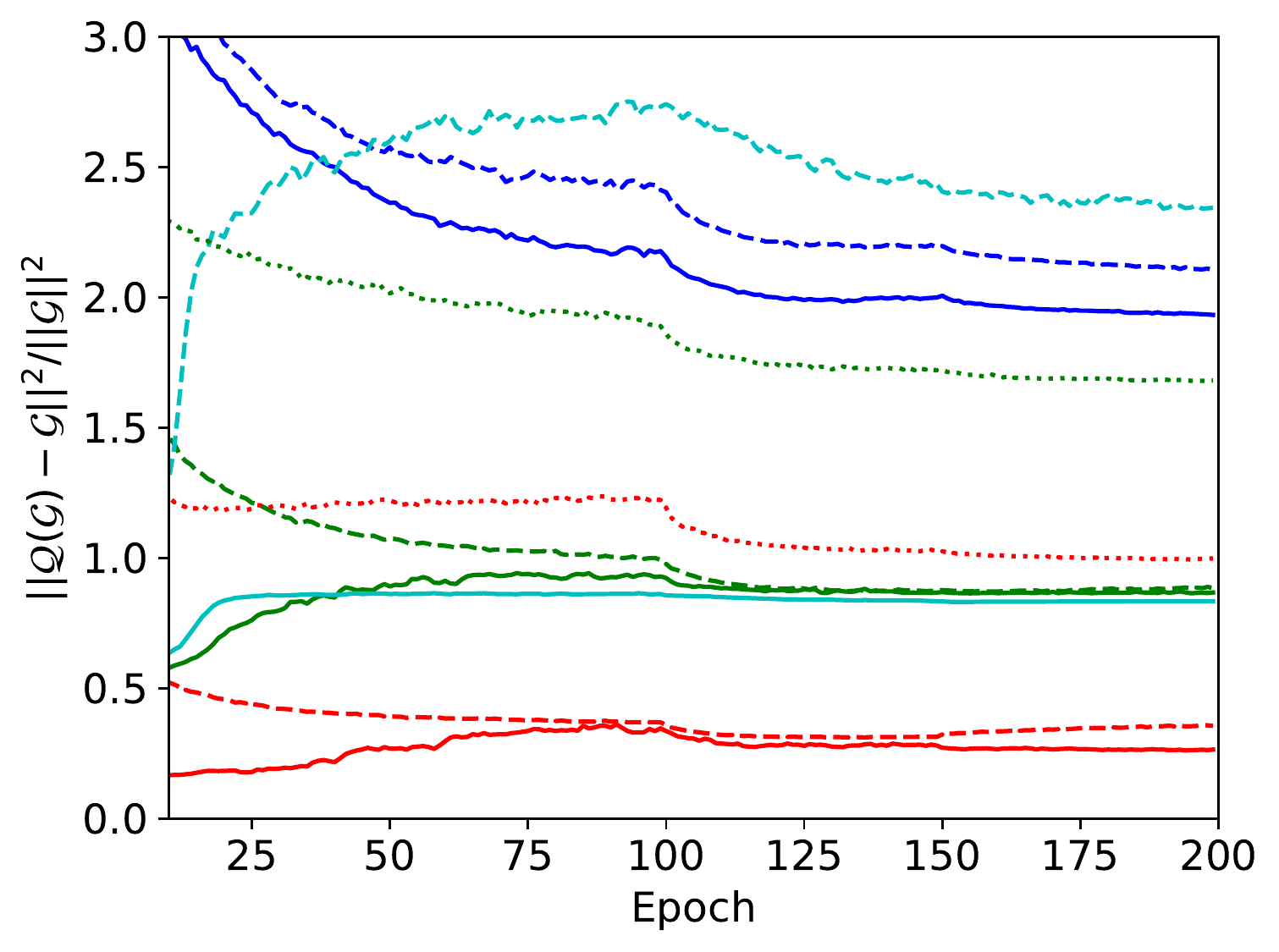}
    \includegraphics[width=.32\linewidth]{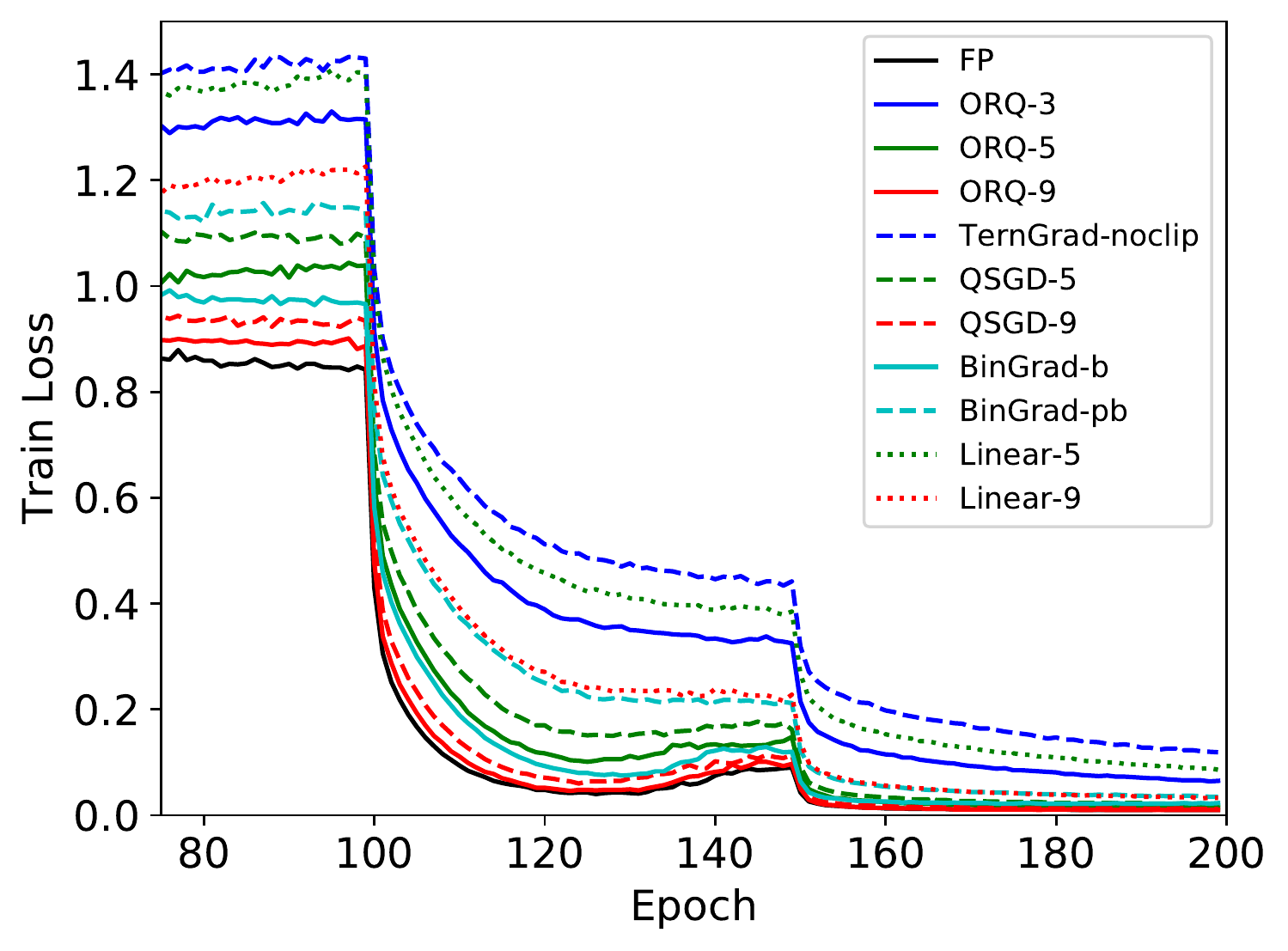}
    \includegraphics[width=.32\linewidth]{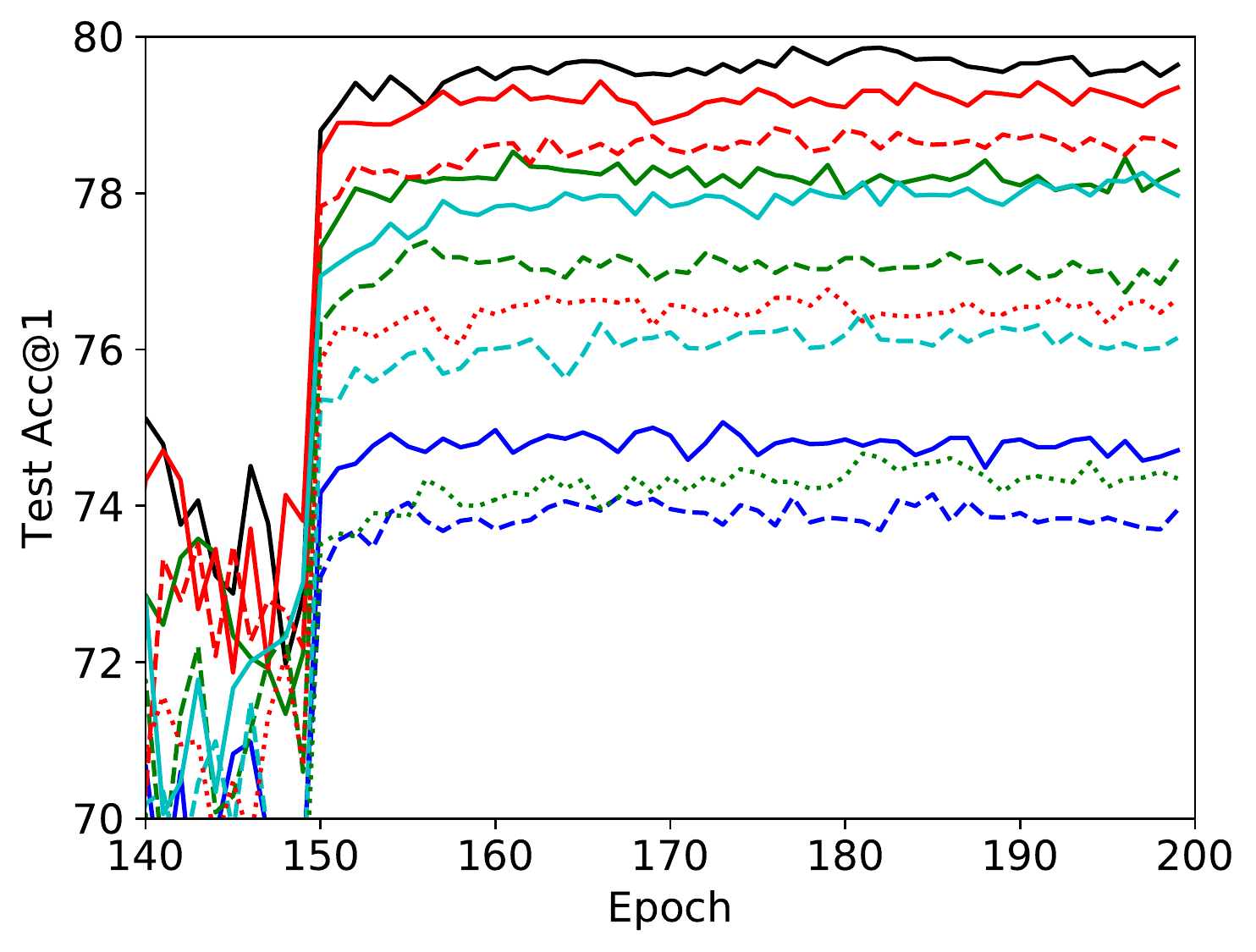}
    \includegraphics[width=.32\linewidth]{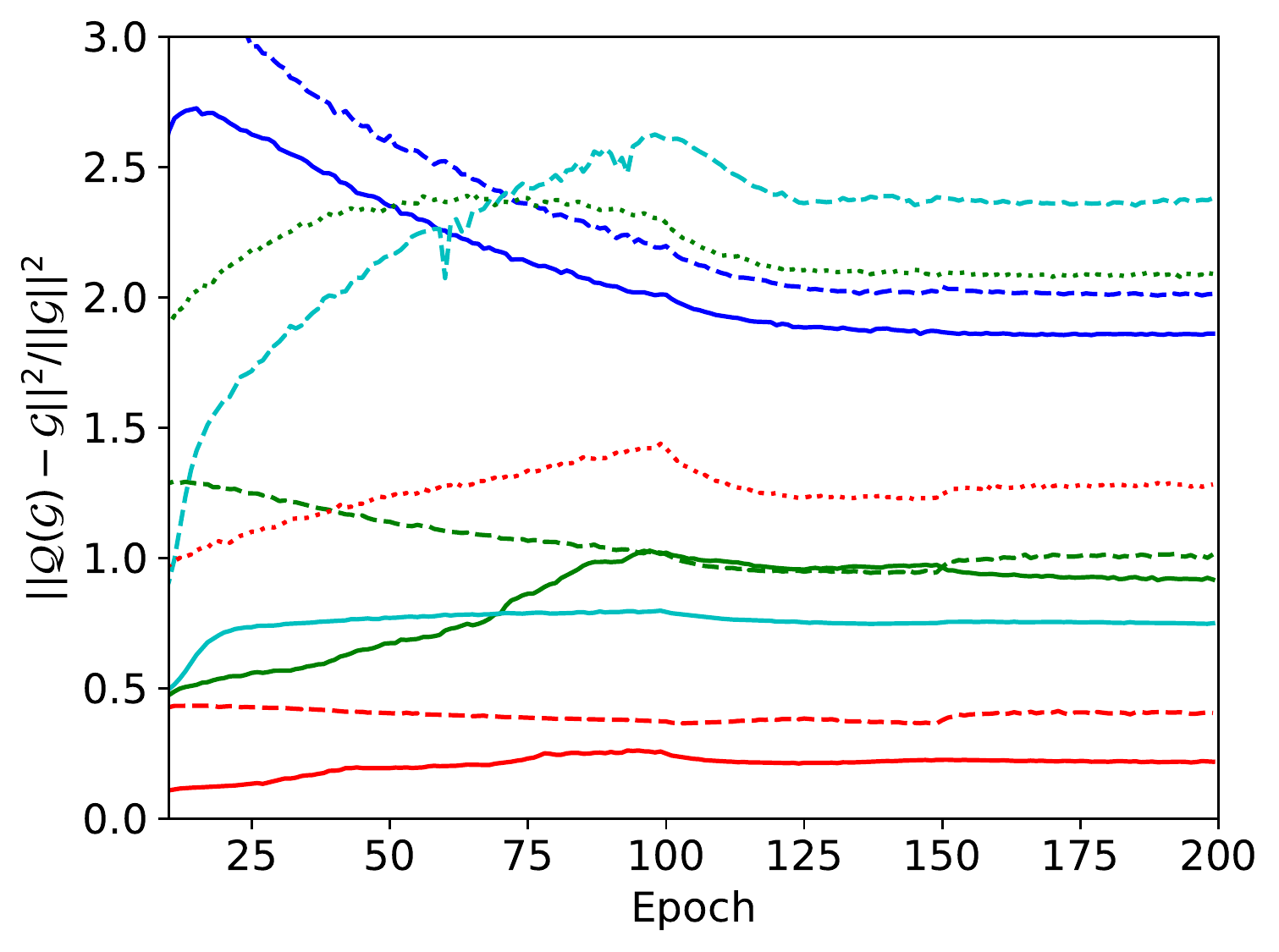}
    \caption{Training curves (CIFAR-100) of different methods. The model architectures from the first to the last row: ResNet-56, ResNet-110, GoogLeNet.}
    \label{cifar training}
\end{figure}

\begin{table}[t]
\caption{CIFAR-100 single worker testing accuracy (\%).}
\centering
\begin{tabular}{lcccc}
\toprule
Compression ratio & Method & ResNet-56 & ResNet-110 & GoogLeNet\\
\midrule
x1 & FP & \underline{72.66} & \underline{72.98} & \underline{79.86}\\
\midrule
\multirow{3}{*}{x32} & BinGrad-pb & 67.20 & 68.58 & 76.47\\
& BinGrad-b & \textbf{68.53} & \textbf{69.83} & \textbf{78.26} \\
& SignSGD & 67.81 & 68.21 & 75.53\\
\midrule
\multirow{2}{*}{x20.2} & TernGrad-noclip & 64.35 & 65.48 & 74.15\\
& ORQ-3 & \textbf{67.02} & \textbf{67.50} & \textbf{75.07}\\
\midrule
\multirow{3}{*}{x13.8} & QSGD-5 & 69.42 & 69.26 &77.38\\
& ORQ-5 & \textbf{70.34} & \textbf{70.57} & \textbf{78.53}\\
& Linear-5 & 66.93 & 66.31 & 74.67\\
\midrule
\multirow{3}{*}{x10.1} & QSGD-9 & 69.52 & 72.14 & 78.83\\
& ORQ-9 & \textbf{72.37} & \textbf{72.50} & \textbf{79.43}\\
& Linear-9 & 68.62 & 67.44 & 76.77\\
\bottomrule
\end{tabular}
\label{cifar100 single worker test acc}
\end{table}

\begin{table}[t]
\caption{CIFAR-10 testing accuracy (ResNet-110, d=512) regarding various bucket size.}
\centering
\begin{tabular}{ccccccccc}
\toprule
Method & 128 & 512 & 1024 & 2048 & 4096 & 8192 & 16384 & 32768 \\
\midrule
TernGrad-noclip & 90.86 & 90.23 & 89.58 & 89.37 & 88.08 & 87.87 & 86.75 & 85.63 \\
ORQ-3 & \textbf{91.66} & \textbf{90.92} & \textbf{90.73} & \textbf{90.18} & \textbf{88.99} & \textbf{88.86} & \textbf{87.70} & \textbf{87.08} \\
\bottomrule
\end{tabular}
\label{ciar10 bucket size test acc}
\end{table}

\begin{table}
\caption{Test accuracy (\%) regarding clipping factor (ResNet-110, d=512).}
\centering
\begin{tabular}{cccccc}
\toprule
Method & CIFAR & $c=1.7$ & $c=2.5$ \\
\midrule
\multirow{2}{*}{ORQ-3} & 10 & 93.05 (-0.35) & 92.26 (-1.14) \\
& 100 & 72.30 (-0.68) & 70.62 (-2.36) \\
\midrule
\multirow{2}{*}{ORQ-5} & 10 & 93.22 (-0.18) & 92.34 (-1.06) \\
& 100 & 72.86 (-0.12) & 71.15 (-1.83) \\
\midrule
\multirow{2}{*}{ORQ-9} & 10 & 93.91 (+0.51) & 92.41 (-0.99) \\
& 100 & 72.84 (-0.14) & 72.19 (-0.80) \\
\bottomrule
\end{tabular}
\label{cifar10 clip factor bucket size 512}
\end{table}

\begin{figure}[t]
    \centering
    \includegraphics[width=.32\linewidth]{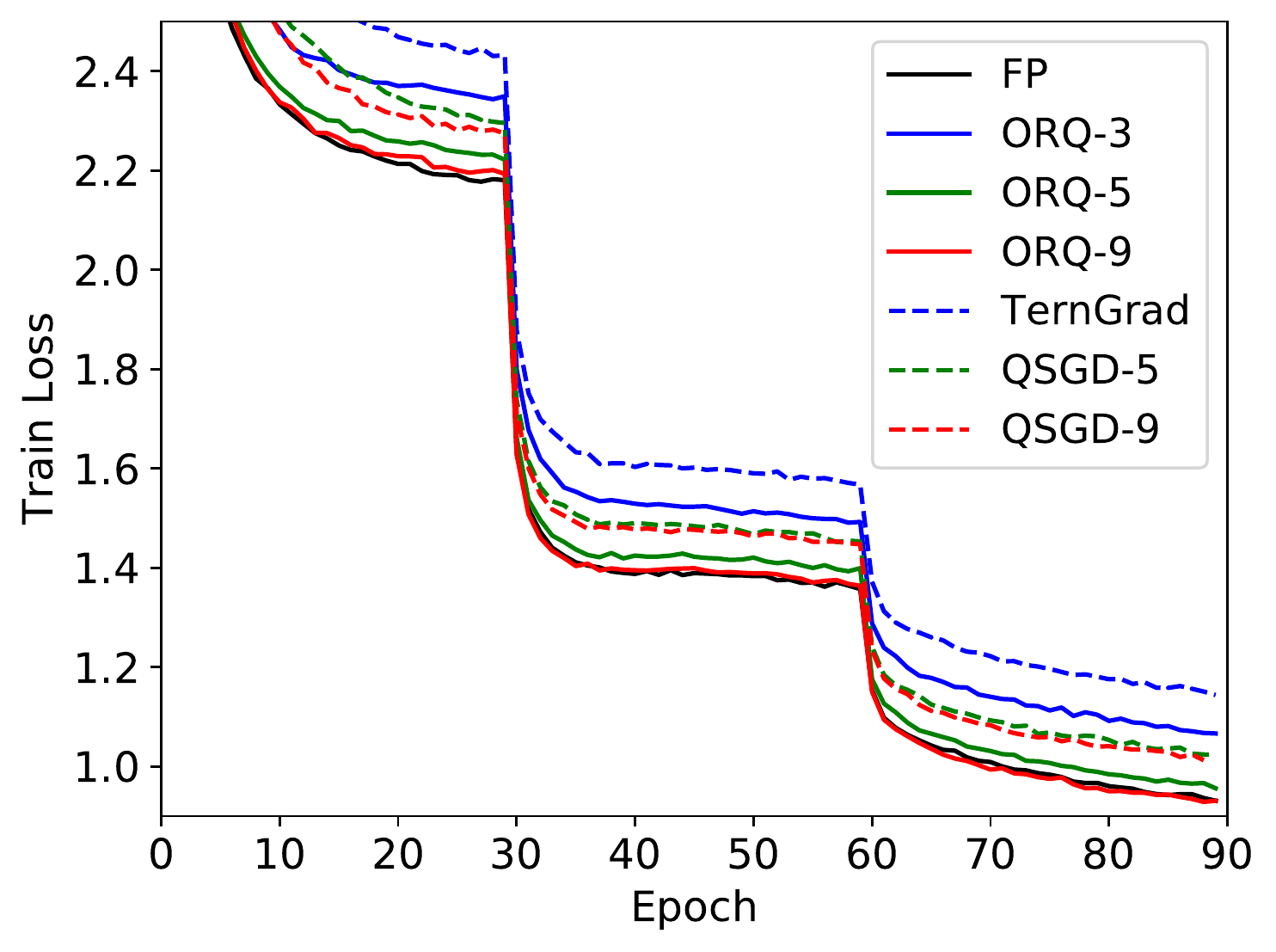}
    \includegraphics[width=.32\linewidth]{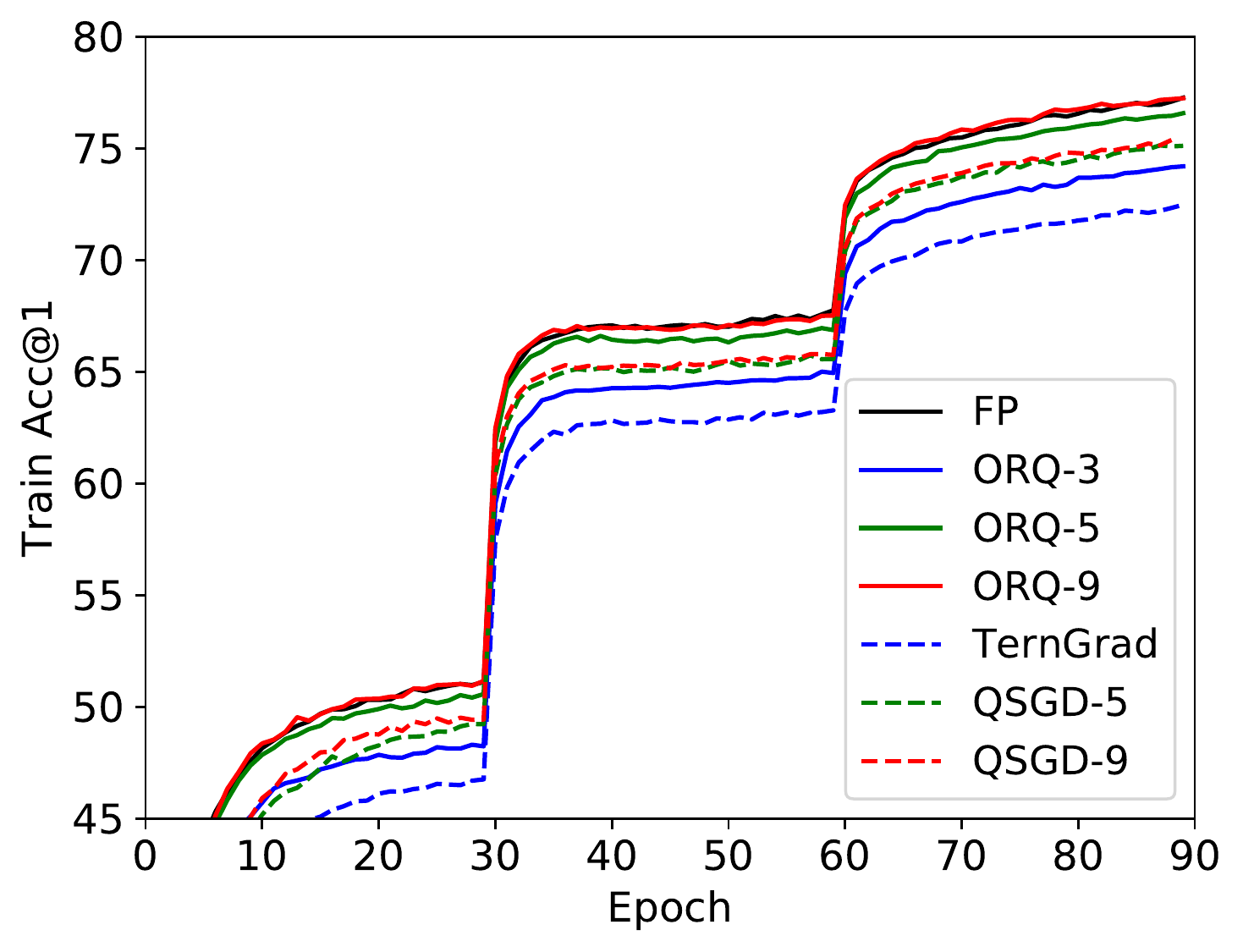}
    \includegraphics[width=.32\linewidth]{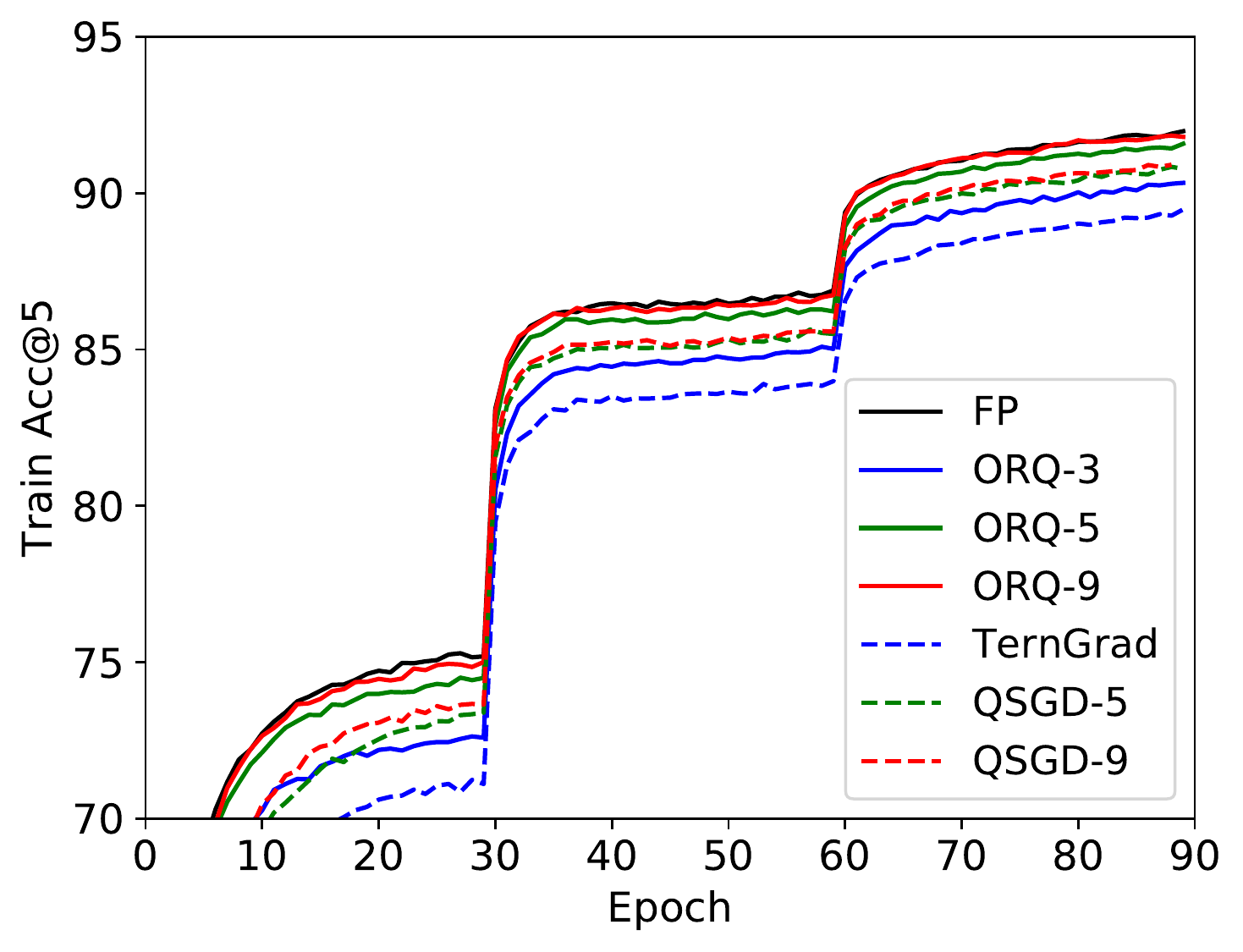}
    \includegraphics[width=.32\linewidth]{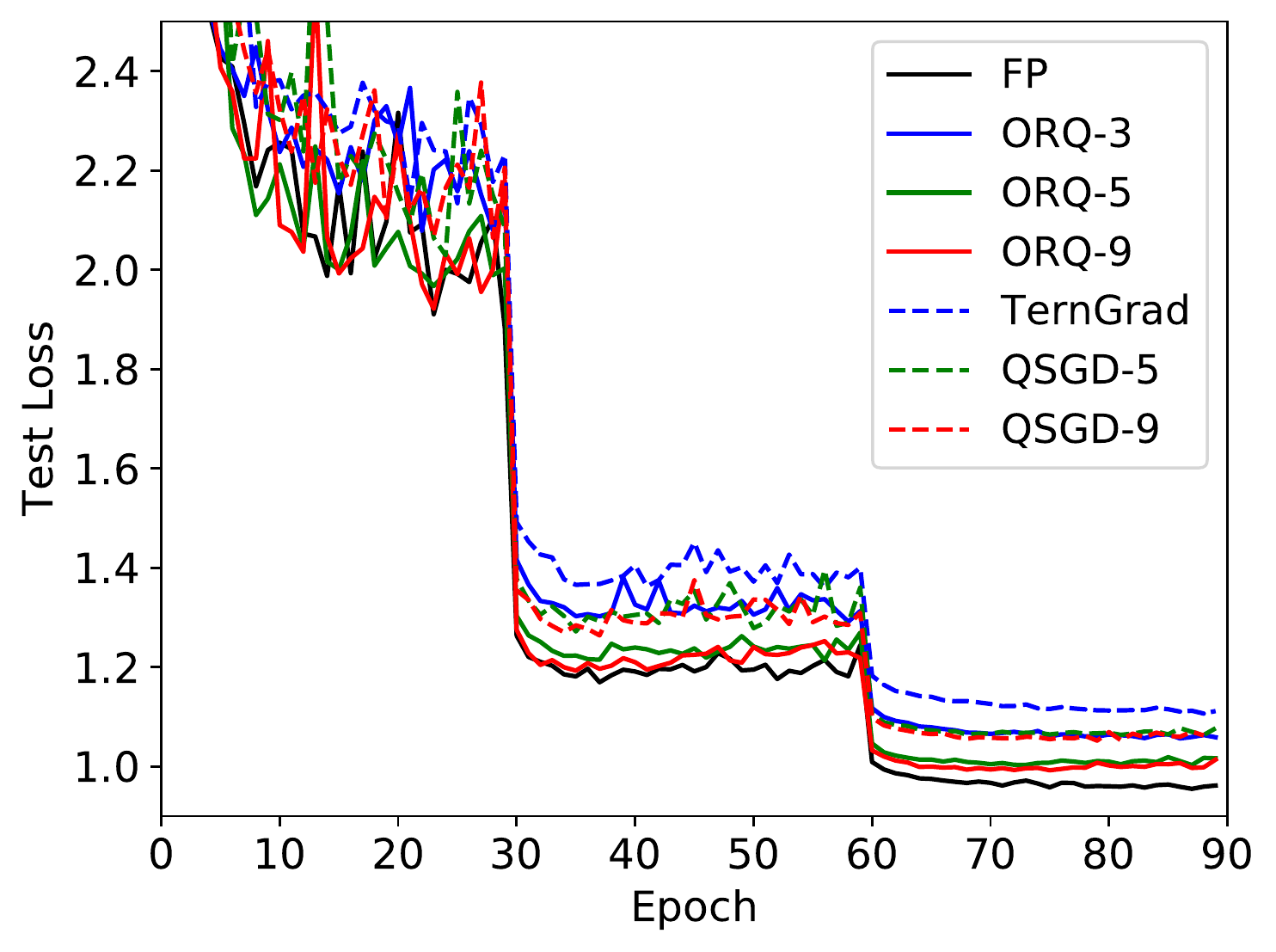}
    \includegraphics[width=.32\linewidth]{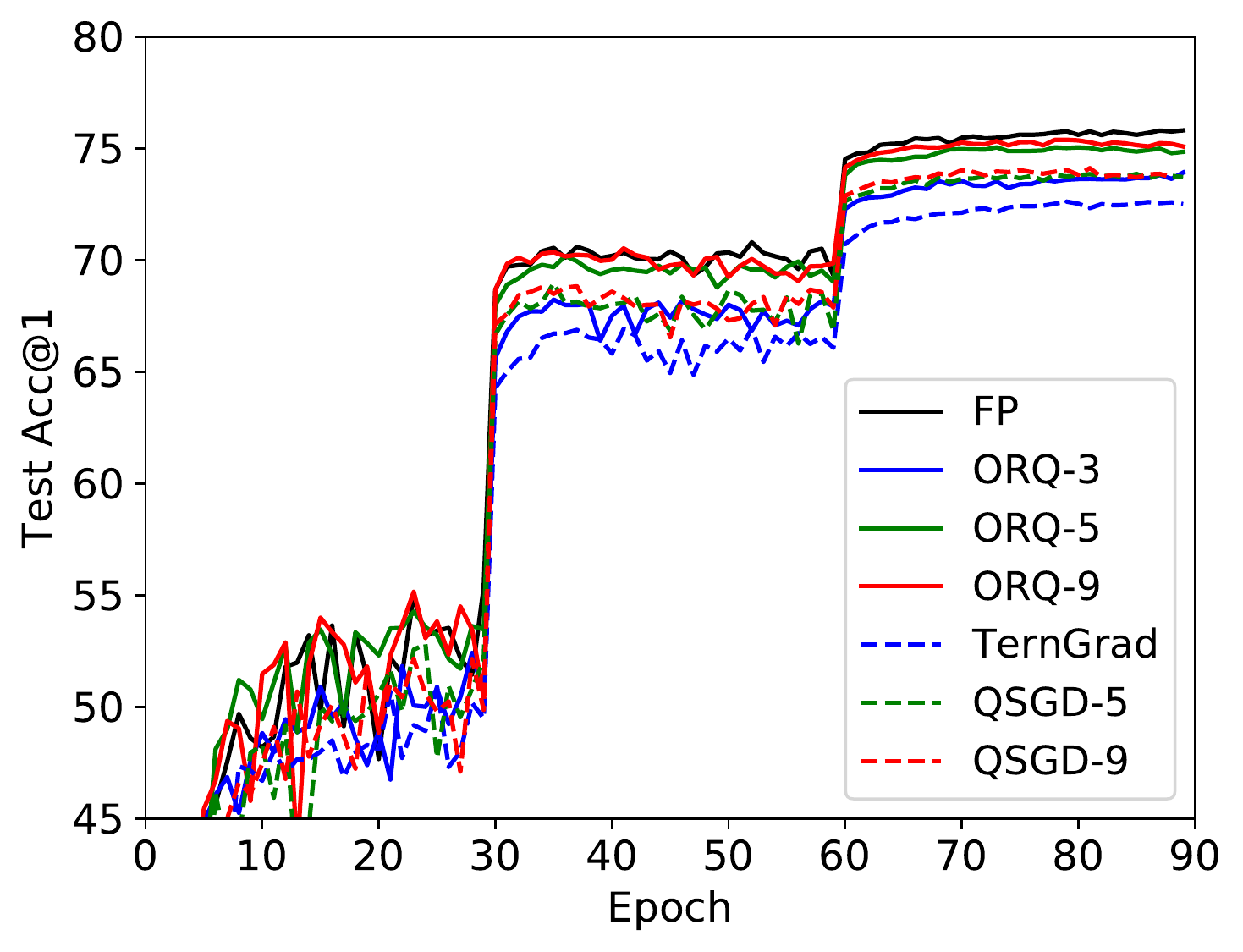}
    \includegraphics[width=.32\linewidth]{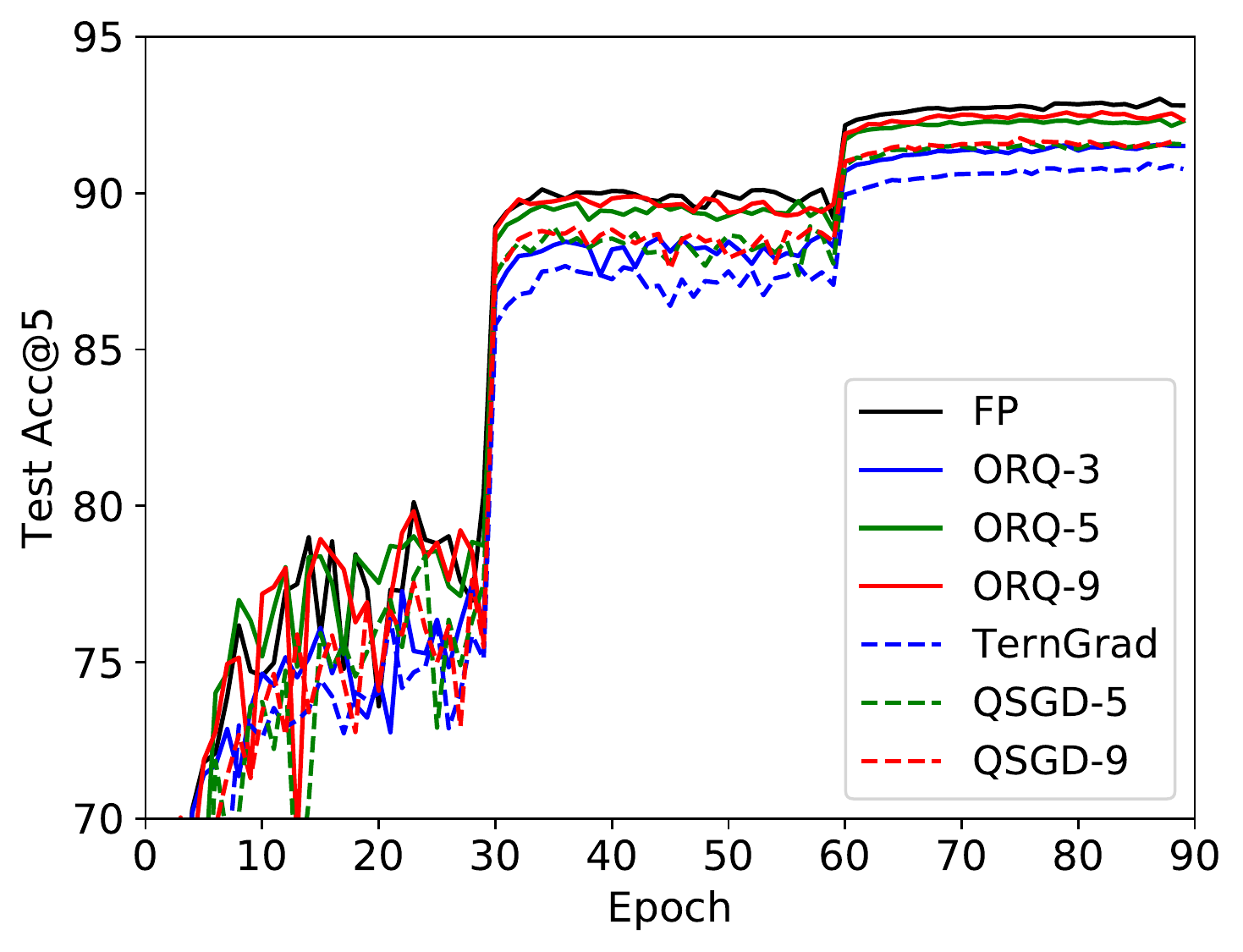}
    \caption{Training curves (ImageNet, ResNet-50) of different methods.}
    \label{imagenet training}
\end{figure}

\begin{table}[t]
\caption{ImageNet test top-1/5 accuracy (\%) of ResNet-50.}
\centering
\begin{tabular}{lcccc}
\toprule
Compression ratio & Method & Top-1 testing accuracy & Top-5 testing accuracy\\
\midrule
x1 & FP & 75.82 & 92.80\\
\midrule
\multirow{2}{*}{x20.2} & TernGrad & 72.62 (-3.20) & 90.70 (-2.10)\\
& ORQ-3 & \textbf{73.92 (-1.90)} & \textbf{91.51 (-1.39)}\\
\midrule
\multirow{2}{*}{x13.8} & QSGD-5 & 73.80 (-2.02) & 75.05 (-0.77)\\
& ORQ-5 & \textbf{75.05 (-0.77)} & \textbf{92.24 (-0.56)}\\
\midrule
\multirow{2}{*}{x10.1} & QSGD-9 & 74.12 (-1.70) & 91.65 (-1.15) \\
& ORQ-9 & \textbf{75.39 (-0.43)} & \textbf{92.58 (-0.22)} \\
\bottomrule
\end{tabular}
\label{imagenet acc}
\end{table}

To demonstrate the effectiveness of BinGrad and ORQ, we conduct experiments on the CIFAR-10/100 \cite{krizhevsky2009learning} and the ImageNet \cite{russakovsky2015imagenet} datasets. Layer-wise quantization and bucket-based quantization are two commonly used methods. The first one quantizes the gradient in each layer independently, while the latter one evenly divides the whole gradient into buckets of the same length $d$ and quantizes each bucket independently. In all our experiments the quantization is bucket-based for a fair comparison, but layer-wise quantization could also be applied. Here we are focused on the performance gain resulting from better quantization levels. Six methods are compared in this section:

\begin{itemize}
    \item TernGrad as proposed in \cite{wen2017terngrad} with 3 quantization levels.
    \item QSGD-s as proposed in \cite{alistarh2017qsgd} with $s$ ($\geq 3$) quantization levels. QSGD-3 is similar to TernGrad.
    \item Linear-s: A naive method choosing $s$ quantization levels by linearly dividing the gradient cumulative distribution. It has also been used to quantize model weights in \cite{han2015deep}.
    \item ORQ-s for multi-level ($s\geq 3$) quantization that we proposed in section \ref{orq}.
    \item BinGrad-pb/BinGrad-b for 2-level quantization that we proposed in section \ref{bingrad}.
    \item Scaled SignSGD.
\end{itemize}

TernGrad proposes a useful gradient clipping technique to be applied before quantizing the gradient: $clip(v)=sign(v)\cdot \min(|v|, c\cdot\sigma)$, where $\sigma^2$ is the gradient variance and $c$ is a positive constant (empirically set to 2.5). We also apply this in our BinGrad and ORQ methods with a linear warm-up schedule starting from base learning rate/10 for 5 epochs. We compare these methods both with and without the gradient clipping technique.

\subsection{CIFAR}

\subsubsection{Implementation details}

We implement all the methods in PyTorch \cite{paszke2019pytorch} and run the experiments on a cluster with Nvidia Tesla P40 GPUs. CIFAR-10/100 experiments are conducted in a single machine environment meaningful for computation and memory efficiency. ResNet-56, ResNet-110 and GoogLeNet models are tested. An SGD optimizer with a momentum constant of 0.9 is applied, the weight decay is set to be 0.0005 and the batch size is 128. We train each model for 200 epochs with a learning rate decay of 0.1 at epoch 100 and 150. The base learning rate is 0.1. Standard data augmentation techniques are incorporated. We did not apply gradient clipping for all methods in CIFAR experiments, because CIFAR is a comparatively simple dataset and we want the influence of different quantization schemes to be clearly reflected. The buckets size $d=2048$ unless stated otherwise.

\subsubsection{Quantization levels}

We first illustrate the different quantization strategies in Figure \ref{levels distribution}. As shown by the FP gradient, the distribution does not resemble Gaussian with sharp edges around zero. Evenly spaced quantization intervals as used in QSGD favor the FP gradient with uniform distribution. They are not a good alternative method of FP gradient due to two additional reasons:

\begin{enumerate}
    \item Low overall utilization of quantization levels. As shown in the first figure of Figure \ref{levels distribution}, very few gradients are quantized to levels away from zero. Such a problem will inevitably lead to a large quantization error because of the decreased gradient diversity and representability.
    \item The loss of gradient shape information. A considerable distortion can be observed in QSGD-9 compared with the FP gradient. An advanced quantization scheme should preserve the shape information as much as possible.
\end{enumerate}

In comparison, the naive Linear gradient quantization method tries to improve the utilization of levels other than zero. However, the quantized gradients of the Linear scheme are mostly centralized around zero, where the density of the gradient distribution is high. Although the Linear strategy tries to balance the number of values quantized to different levels, it lost most of the gradient shape information which may degrade the test performance. To achieve better performance with the quantized gradient, our proposed methods improve both these two criteria with a balance. Through the greedy determination of quantization levels with the optimal condition, our methods ORQ-9 as in the third figure of Figure \ref{levels distribution} not only maintains a much better gradient shape distortion but also improves the utilization of quantization levels other than zero.

For the ultra low-level quantization (2/3 levels), there is little space left for us to optimize the selection of the quantization levels. In our proposed BinGrad-b and BinGrad-pb, there is a trade-off between bias and variance when considering which method to choose. BinGrad-b achieves minimum variance/quantization error and is the first to consider in the single machine environment. However, BinGrad-pb achieves reduced bias with enlarged quantization range as shown in the fourth figure of Figure \ref{levels distribution} at the cost of larger quantization error. In distributed settings, the variance could be reduced by averaging the gradient from different workers.

\subsubsection{Performance}

To validate the analysis of quantization error, we record it during training with different methods as shown in Figure \ref{cifar training}. The quantization error of our proposed ORQ method always surpasses its counterparts of the same number of quantization levels. ORQ performs better especially when the quantized gradient uses fewer quantization levels. This gap also contributes to better convergence and test performance. In overall, the quantization error of BinGrad-b is less than BinGrad-pb. BinGrad achieves better test performance and smaller quantization error than 3-level quantization because itself essentially can be regarded as a clipping scheme that removes large gradient values. If to incorporate gradient clipping, the 3-level quantization methods can achieve better performance which is very close to FP gradients as shown in Table \ref{cifar10 clip factor bucket size 512}.

Regarding top-1 testing accuracy (Table \ref{cifar100 single worker test acc}), ORQ achieves an improvement ranging from 0.36\% to 2.85\% under the no clipping, relatively large batch size and single machine settings compared with its counterparts, while naive Linear methods can lead to even more degradation. A smaller bucket size can contribute to less gradient information loss as shown in Table \ref{ciar10 bucket size test acc}. It also shows that ORQ is more resilient to the increasing bucket size. With bucket size growing from 128 to 32768, the test accuracy of ORQ degrades 4.58\%, while its counterpart degrades 5.23\%. A larger bucket size is more desirable to decrease the cost of sending floating-point to represent quantization levels. But it is not a major factor affecting the compression ratio as the gain from quantizing each gradient element is more overwhelming.

\subsection{ImageNet}

\subsubsection{Implementation details}

We run all the ImageNet experiments with multiple workers to test the performance of unbiased schemes under distributed settings with reduced gradient precision. We use an SGD optimizer with a momentum constant of 0.9. The weight decay is 0.0001 and the batch size is 256 in total. A mini-batch is evenly split onto 4 workers. The statistical parameters (e.g., the statistics of the batch normalization layer) are broadcast from the first worker to the other workers during training. We train the model for 90 epochs with a learning rate decay of 0.1 at epoch 30 and epoch 60. The base learning rate is 0.1. Standard data augmentation techniques are incorporated. All the experiments use the bucket size $d$ of 512.

\subsubsection{Performance}

We show the metrics during the training process in Figure \ref{imagenet training} and the final test performance in Table \ref{imagenet acc}. ORQ-5 and ORQ-9 can achieve very close performance regarding training metrics and comparable performance regarding testing metrics compared with FP gradient. In terms of test accuracy and under the distributed settings, ORQ achieves an improvement of around 1.3\% and 0.8\% for top-1 and top-5 testing accuracy respectively compared with the counterpart. In particular, decreasing the compression ratio from 20.2 to 10.1 leads to a top-1 testing accuracy gain of 1.47\% for ORQ, but only 0.95\% for the counterpart. Moreover, ORQ-3 achieves very similar test performance compared with QSGD-5 and QSGD-9. Under the same test performance requirements, we can apply a more aggressive gradient quantization precision for lower communication overheads using ORQ.

\section{Conclusion}

In this paper, we deduced the optimal condition for communication-efficient distributed training of deep neural networks with the binary and multi-level compressed gradient. The optimal condition has no extra assumption on the gradient distribution, theoretically making it the most suitable scheme for quantizing the gradient of the neural network's any layer at any training stage. BinGrad-b/pb and ORQ are proposed based on the optimal quantization condition. Experimental results show the effectiveness of BinGrad and ORQ without gradient clipping on CIFAR-10/100 datasets in the single machine environment, and ORQ with gradient clipping on ImageNet in the distributed environment. With gradient clipping, ORQ achieves close performance compared with the FP gradient. In comparison, naive Linear methods may not help improve the performance. As future works, the greedy algorithm for determining the quantization levels in ORQ may be further improved.

\bibliographystyle{abbrv}  
\bibliography{main.bib}  

\end{document}